\documentclass{article}
\usepackage{ijcai24}

\usepackage{times}
\usepackage{soul}
\usepackage{url}
\usepackage[hidelinks]{hyperref}
\usepackage[utf8]{inputenc}
\usepackage[small]{caption}
\usepackage{graphicx}
\usepackage{amsmath}
\usepackage{amsthm}
\usepackage{booktabs}
\usepackage{algorithm}
\usepackage{algorithmic}
\usepackage[switch]{lineno}

\usepackage{caption}
\usepackage{subcaption}
\usepackage{xcolor}
\usepackage{amssymb}
\usepackage{multirow}

\newtheorem{theorem}{Theorem}
\newtheorem{lemma}{Lemma}
\newtheorem{assumption}{Assumptions}
\newtheorem{definition}{Definition}
\theoremstyle{remark}
\newtheorem{remark}{Remark}

\title{Physics-Informed Neural Networks: Minimizing Residual Loss with Wide Networks and Effective Activations}

\author{
Nima Hosseini Dashtbayaz$^1$
\and
Ghazal Farhani$^{2,}$\footnote{Corresponding Author}\and
Boyu Wang$^{1,3,*}$\And
Charles X. Ling$^1$
\affiliations
$^1$Department of Computer Science, University of Western Ontario\\
$^2$National Research Council of Canada\\
$^3$Vector Institute\\
\emails
nhosse5@uwo.ca,
ghazal.farhani@nrc-cnrc.gc.ca,
bwang@csd.uwo.ca,
charles.ling@uwo.ca
}

\begin{document}

\maketitle

\begin{abstract}

    The residual loss in Physics-Informed Neural Networks (PINNs) alters the simple recursive relation of layers in a feed-forward neural network by applying a differential operator, resulting in a loss landscape that is inherently different from those of common supervised problems. Therefore, relying on the existing theory leads to unjustified design choices and suboptimal performance. In this work, we analyze the residual loss by studying its characteristics at critical points to find the conditions that result in effective training of PINNs. Specifically, we first show that under certain conditions, the residual loss of PINNs can be globally minimized by a wide neural network. Furthermore, our analysis also reveals that an activation function with well-behaved high-order derivatives plays a crucial role in minimizing the residual loss. In particular, to solve a $k$-th order PDE, the $k$-th derivative of the activation function should be bijective. The established theory paves the way for designing and choosing effective activation functions for PINNs and explains why periodic activations have shown promising performance in certain cases. Finally, we verify our findings by conducting a set of experiments on several PDEs. Our code is publicly available at \href{https://github.com/nimahsn/pinns_tf2}{\url{https://github.com/nimahsn/pinns\_tf2}}.
\end{abstract}

\section{Introduction}
The success of deep learning in a wide variety of tasks has motivated its application in scientific domains as well \cite{sirignano2018dgm,reiser2022graph,li2020fourier}. PINNs \cite{raissi2017physics} in particular are designed to solve differential equations as an alternative to traditional solvers, benefiting from discretization-free construction and the vast availability of machine learning tools and techniques. As a result, PINNs have been deployed in various physics and engineering problems, such as solving inverse scattering problems in photonics \cite{chen2020physics}, flow problems in fluid dynamics \cite{cai2021physics}, and computational neuromusculoskeletal models raised in biomedical and rehabilitation sciences \cite{zhang2022physics}.

Consider a general-form PDE with a Dirichlet boundary condition such as
\begin{equation}
\label{eq: pde}
    \begin{aligned}
         &\mathcal{D}\left[u\right](x) = f(x)\; &x\in\Omega&\\
         &u(x) = g(x)\; &x\in\partial\Omega&,
    \end{aligned}
\end{equation}
where $u$ is the solution of the PDE on a bounded domain \mbox{$\Omega\subset\mathbb{R}^d$} of $d$ independent variables with boundaries $\partial \Omega$, $f$ and $g$ are known functions, and $\mathcal{D}$ is a differential operator. Here, the operator $\mathcal D$ expresses the physical rules governing $u$ through a differential expression. PINNs are then trained to respect the underlying physical dynamics given in $\mathcal{D}$ by minimizing the residual loss
\begin{equation}\label{eq: residual loss}
    L_r = \sum_{x\in \mathbf{x}}l(\mathcal{D}\left[\hat{u}\right](x) - f(x)),
\end{equation}
where $\hat{u}$ is a neural network approximation of $u$, $l$ is an error function such as squared error, and $\mathbf{x}$ is a set of training collocation points in $\Omega$. To guarantee a unique solution, boundary (and initial) conditions are also imposed by adding other supervised loss terms, referred to as boundary loss, trained with boundary data sampled from $\partial \Omega$. The resulting loss function can then be treated as a multi-objective optimization task \cite{raissi2017physics}.

While proven effective, training PINNs is often a challenging task. These challenges usually originate from either the discrepancy between the residual loss and the boundary loss \cite{wang2020pinns,farhani2022momentum,wang2020understanding} or the nature of the residual loss \cite{krishnapriyan2021characterizing,wang2022respecting}. Notably, as Eq. \ref{eq: residual loss} involves differentiation over a neural network with respect to (w.r.t.) its inputs, the outputs of the network undergo a significant structural transformation. 
To better understand the aforementioned process and its implications, let us consider a simple differential operator $\frac{\partial u}{\partial x}$ of a single independent variable and an $L$-layer feed-forward network with an activation function $\sigma$. One can find that the application of this differential operator on the neural network, $\mathcal{D}[\hat u]$, is given by
\begin{equation}\label{eq: simple D}
\mathcal D \left[\hat u\right] = W_L^\top \times (\sigma'(G_{L-1})\circ W_{L-1})^\top \times \dots \times(\sigma'(G_1)\circ W_1)^\top,
\end{equation}
where $W_i$ and $G_i$ are weights and linear outputs of layer $i$, and $\circ$ and $\times$ denote element-wise (Hadamard) and matrix products. In contrast, the original neural network $\hat u$ can be defined recursively as 
$$
\hat{u}(x) = G_L(x),\; G_i(x) = \sigma(G_{i-1}(x))\times W_{i} + b_i.
$$

Eq. \ref{eq: simple D} shows how the differentiation transforms the outputs of a neural network. Firstly, note that the simple recursive relation between the layers of a feed-forward network is disrupted by applying $\mathcal D$, and additional element-wise products with weights emerge as well. More significantly, we observe that the derivative $\sigma'$ of the activation function appears in the outputs. This presence of $\sigma'$ in PINNs highlights the importance of an activation function with well-behaved derivatives in the model's expressive power in learning $\mathcal D$ and likewise in the optimization process as it involves higher-order derivatives of $\sigma$.

Altogether, the distinct characteristics of $\mathcal{D}[\hat u]$ and the resulting residual loss, contribute to a problem that is quite different from common supervised training tasks. Consequently, the existing theory around loss functions and their characteristics cannot readily be applied to PINNs, and the lack of understanding about PINNs and their optimization dynamics leads to uninformed design choices and suboptimal performance even for seemingly easy PDEs.

In this work, we focus on the residual loss and its landscape. Specifically, we are interested in finding what neural networks and design choices enable PINNs to globally minimize the residual loss. To this end, we study the residual loss at a critical point of the network parameter space and search for distinctive characteristics of a global minimum compared to other critical points. Once these characteristics are identified, our investigation shifts towards determining sufficient conditions within the network design, in particular, width and activation function, that guarantee the existence of global minima in the parameter space. Our findings underscore the importance of the width and activation functions with well-behaved high-order derivatives in acquiring a high expressive power in learning the differential operator. Finally, we verify our findings by conducting extensive experiments on several PDEs.   

We summarize our contributions as follows. (1) We theoretically show that the residual loss of PINNs can be globally minimized, given a two-layer neural network with a width equal to or greater than the number of collocation points. (2) Through our analysis, we establish that the residual loss for a $k$-th order differential operator is optimally minimized when using an activation function with a bijective $k$-th order derivative. We leverage this theoretical foundation as a guideline for selecting activation functions, justifying the choice of sinusoidal activations, and subsequently validating their effectiveness through empirical demonstrations and experiments.

\section{Related Works}
\subsection{Wide Neural Networks}
Wide neural networks have historically been of significant interest in machine learning. With classical results such as Universal Approximation and Gaussian processes, and more recently, NTK theory \cite{jacot2018neural}, wide networks have been studied to understand neural networks in certain regimes \cite{chen2020generalized,lee2019wide}. The optimal width of a neural network is also studied for convergence guarantees \cite{oymak2020toward,du2019gradient,allen2019convergence,nguyen2020global} and loss geometry \cite{safran2016quality,nguyen2017loss} with certain classes of neural networks and optimizers. The convergence guarantees are often provided for a width polynomial in the number of training samples and the number of layers \cite{allen2019convergence}. Safran and Shamir \shortcite{safran2016quality} studied the basins of the loss function for wide two-layer ReLU networks, showing that wider networks are initialized at a good basin with higher probability. \cite{nguyen2017loss} also showed that most of the critical points in a wide neural network are also global minima. The developed theory in most of the aforementioned works cannot be directly applied to PINNs, as they either rely on specific neural network formulations \cite{nguyen2017loss,nguyen2020global}, which are disrupted by differentiation, or certain hyper-parameters that are not effective for PINNs, such as ReLU activation function \cite{du2019width,safran2016quality,allen2019convergence}.

\subsection{Periodic Activation Functions}
Sitzmann et al. \shortcite{sitzmann2020implicit} proposed using sinusoidal activation functions in neural networks with low-dimensional inputs for learning differentiable signals. Notably, they also showed the capability of Sine networks in solving Wave and Helmholtz PDEs with PINNs. Since then, few works have explored the behaviour of neural networks with periodic activation functions at initialization \cite{belbute2022simple} and their expressive power as function approximators \cite{meronen2021periodic}. Meronen et al. \cite{meronen2021periodic} studied the inductive bias introduced by periodic activation functions on the neural network functional space, and showed that such networks are less sensitive to input shifts.

\subsection{Physics-Informed Neural Networks}
Besides the applications of PINNs in solving various PDEs, there has been a surge in analyzing the behaviour and pitfalls of PINNs in recent years, especially from the optimization perspective \cite{wang2020understanding,liu2020linearity,farhani2022momentum}. Using \textit{Neural Tangent Kernel} (NTK) theory from infinitely wide neural networks, \cite{wang2020pinns} showed that high-frequency terms in a PDE result in discrepancies in the convergence rate between the loss objectives when trained with Gradient Descent, leading the model to exhibit behaviours similar to spectral bias \cite{rahaman2018spectral}. Wang et al. \shortcite{wang2020understanding} also showed similar results by studying the magnitude of the loss gradients at different layers. Later on, the momentum term was shown to address the discrepancy in optimization in the infinite-width regime \cite{farhani2022momentum}.  

Many recent works alleviate the optimization challenges in PINNs and improve their performance by assigning weights to each loss term \cite{wang2020pinns,mcclenny2020self,wight2020solving}, designing new architectures and embeddings \cite{wong2022learning,wang2020understanding,wang2021eigenvector,dong2021method}, and using sophisticated training strategies such as curriculum learning \cite{krishnapriyan2021characterizing,wang2022respecting}. Among them, \cite{wang2020pinns} and \cite{wong2022learning}, suggested mapping the inputs to random or trainable Fourier features and the use of sinusoidal activation functions to overcome the spectral bias and the convergence discrepancy. 

\section{Global Minima of the Residual Loss} \label{sec: section 3}
 In this section, we study the residual loss at its critical points to obtain sufficient conditions for the existence of global minima.  We present the proofs for the lemmas and theorems in this section in the Appendix. First, we introduce the notation and the setup used throughout this section.
\subsection{Notation and Setup}
We use $\hat{u}_{\mathcal{W}}: \mathbb{R}^d\times \mathbb{R}^{|\mathcal{W}|} \rightarrow \mathbb{R}^{n_L}$ to denote an $L$-layer feed-forward neural network parameterized by $\mathcal{W} = \{W_i, b_i\mid 1\leq i\leq L, W_i\in \mathbb{R}^{n_{i-1} \times n_i},b_i \in \mathbb{R}^{n_i}\}$, where $n_i$ is the number of neurons in layer $i$, $n_0 = d$, and $n_L = 1$. We drop $\mathcal{W}$ from $\hat u_\mathcal{W}$ for simplicity if there is no ambiguity. The neural network $\hat u$ for an input $x=(x_1, \dots, x_d)$ is formulated as
\begin{equation}\label{eq: nn}
    \begin{aligned}
        &\hat u(x) = G_L(x),&\;\\
        &G_i(x) = F_{i-1}(x) \times W_i + b_i&\;\forall i \in \{1, \dots L\},\\
        &F_i(x) = \sigma(G_i(x))&\;\forall i \in \{1, \dots L-1\},
    \end{aligned}
\end{equation}
where $\sigma$ is an activation function, and $F_0(x) = x$. We further define $F^{(k)}_i(x)$ as
$$
F^{(k)}_i(x) = \sigma^{(k)}(G_i(x)),
$$
where $\sigma^{(k)}$ is the $k$-th derivative of $\sigma$. In the case of $k=1$, we simply use $F'_i$ and $\sigma'$ instead. For a batch $\mathbf{x}$ of $N$ samples, $F_i(\mathbf{x})$, $G_i(\mathbf{x})$, and $F^{(k)}_i(\mathbf{x})$ are $N\times n_i$ matrices. Also, the matrix power $W^k$ represents an element-wise power.

In a PINN, the neural network $\hat u$ is trained to approximate the solution $u$ of a differential equation denoted as in Eq. \ref{eq: pde}. In this work, we assume that $\mathcal D$ is a linear differential operator, i.e., the PDE is linear in the derivatives of $u$ and $u$ itself. We reformulate the residual loss in Eq. \ref{eq: residual loss} to be a function of weights $\mathcal W$ and choose $l(r) = r^2$.
\begin{equation} \label{eq: phi residual loss}
    \phi_r(\mathbf{x}; \mathcal{W}) = \sum_{x \in \mathbf{x}} l(\mathcal{D}\left[\hat u\right](x) - f(x)).
\end{equation}

Throughout the rest of this section, we consider a two-layer neural network and a simple $k$-th order differential operator $\mathcal{D}\left[u\right] = \frac{\partial^k u}{\partial x^k}$ for a single independent variable $x$ (i.e., $d=1$). We generalize the results in this section to more independent variables in the Appendix.

\subsection{Residual Loss of a Two-layer PINN}
To study the residual loss and its critical points, we first need to derive the analytic formula for the residual loss and its gradients. The next two lemmas, provide us with these tools by finding the differentiation $\mathcal{D}\left[\hat u\right]$ and then deriving the gradients of the resulting residual loss.
\begin{lemma} \label{theorem: 2 layer Du}
    For a two-layer neural network $\hat u$ defined in Eq. \ref{eq: nn}, and a $k$-th order differential operator $\mathcal{D}\left[u\right] = \frac{\partial^k u}{\partial x^k}$ of a single independent variable $x$, $\mathcal{D}[\hat u]$ is 
    \begin{equation*}
    \begin{aligned}
                \mathcal{D}[\hat{u}](x) = W_2^\top \times (F_1^{(k)}(x) \circ W_1^k)^\top.
        \end{aligned}
    \end{equation*}
\end{lemma}

With the analytic formula for $\mathcal{D}[\hat u]$ in hand, it is easy to plug it into Eq. \ref{eq: phi residual loss} to get the residual loss. The next lemma derives the gradients $\nabla_{W_2} \phi_r(\mathbf{x}; \mathcal{W})$ of the residual loss w.r.t. the weights of the last layer.

\begin{lemma}\label{theorem: 2 layer gradients W_L} For $\hat u$ and $\mathcal{D}[\hat u]$ given in Lemma \ref{theorem: 2 layer Du}, gradients of the residual loss w.r.t. the weights of the second layer over the training collocation data $\mathbf x$ of $N$ samples are given by
\begin{equation*}
    \nabla_{W_2}\phi_r(\mathbf x; \mathcal W) =  W_1^k\circ[l'(\mathcal D[\hat u](\mathbf x) - f(\mathbf x))^\top \times F^{(k)}_1(\mathbf x)].
\end{equation*}
\end{lemma}

\begin{remark}
    Lemmas \ref{theorem: 2 layer Du} and \ref{theorem: 2 layer gradients W_L} generalize the appearance of derivatives in the outputs of the neural network as in Eq. \ref{eq: simple D}, showing that a $k$-th order differential term similarly contains the $k$-th derivative of the activation function. Thus, activation functions with vanishing high-order derivatives, such as ReLU, significantly reduce the network representation power in approximating the residuals. Note that the gradients w.r.t. $W_1$ contain the $(k+1)$-th derivative of the activation function, further highlighting the importance of well-behaved derivatives in optimization.
\end{remark}

In the following section, the gradients given in Lemma 2 are studied at a critical point to find the characteristics of global minima of the residual loss. Note that global minimum in this context refers to the parameters that make the loss zero.

\subsection{Critical Points of Wide PINNs}\label{sec: width N and activation}

We are eventually interested in finding sufficient conditions for the existence of a global minimum of the residual loss, i.e., $\phi_r(\mathbf x; \mathcal W) = 0$. The following theorem takes the first step by providing a necessary condition for globally minimizing the residual loss. We then turn this requirement into a sufficient condition by establishing a set of assumptions. Note that the squared error $l(r)$ is a non-negative convex function of the residuals $r$, and $l'(r)=0$ results in $l(r) = 0$. Thus, a critical point $\overline{\mathcal{W}}$ of $\phi_r(\mathbf{x}; \mathcal{W})$ in the parameter space globally minimizes the residual loss if $l'(\mathcal D [\hat u_{\overline{\mathcal W}}](x) - f(x)) =0$ for every training sample in $\mathbf x$. 
\begin{theorem}\label{theorem: full rank two layer}
    For $\hat u$ and $\mathcal{D} [\hat u]$ as in Lemma \ref{theorem: 2 layer Du}, a critical point $\overline{\mathcal W}$ of the residual loss $\phi_r(\mathbf{x}, \mathcal W)$ is a global minimum if the following conditions are satisfied:
\begin{enumerate}
    \item Weights $\overline{W}_1$ of the first layer are strictly non-zero,
    \item $F^{(k)}_1$ has full row rank, i.e., $\text{rank}(F_1^{(k)}(\mathbf x)) = N$.
\end{enumerate}
\end{theorem}

Theorem \ref{theorem: full rank two layer} distinguishes the global minima from other critical points of the residual loss. However, there is no guarantee that an arbitrary neural network can satisfy the conditions in this theorem. In other words, a critical point that makes $F_1^{(k)}(\mathbf x)$ full row rank may not exist in the parameter space of a neural network at all. Still, this theorem does give out a necessary condition for such a neural network. Since $F_1^{(k)}(\mathbf x)$ is an $N \times n_1$ matrix, the width $n_1$ of the first layer must be at least $N$ for it to be full row rank. In fact, given other assumptions, the next theorem shows that $n_1\geq N$ is also a sufficient condition for the existence of a global minimum. Note that the first condition on $W_1$ is satisfied with a high probability in a continuous high-dimensional parameter space.

We first define the non-degenerate critical points used in the next theorem and establish a set of assumptions that connect the two theorems together.

\begin{definition}[Non-degenerate Critical Point \cite{nguyen2017loss}] 
    For a function $f\in C^2:U\subset \mathbb R ^n \rightarrow \mathbb R$ (i.e., $f$ has continuous second-order derivatives), a critical point $x=(x_1, \dots, x_n) \in U$of $f$ is non-degenerate if its Hessian matrix at $x$ is non-singular. Furthermore, $x$ is non-degenerate on a subset of variables $s\subset \{x_1,\dots,x_n\}$ if the Hessian w.r.t. only the variables in $s$ is non-singular at $x$.
\end{definition}

\begin{assumption}\label{assumption: 2 layer single var}
    For the collocation training data $\mathbf{x}$ of $N$ points, the activation function $\sigma$ in $\hat u$, and the $k$-th order differential operator defined in Lemma \ref{theorem: 2 layer Du}, we assume that
    \begin{enumerate}
        \item \label{assumption: distinc samples}samples in $\mathbf{x}$ are distinct,
        \item \label{assumption: bijective} $\sigma ^ {(k)}$ is a continuous and strictly monotonically increasing function, and
        \item $\sigma^{(k)}$ is a bounded function with an infimum of zero.\label{assumption: infimum}
    \end{enumerate}
\end{assumption}
\begin{theorem}\label{theorem: width N}
    With Assumptions \ref{assumption: 2 layer single var} holding and for $\mathcal D \left[\hat u\right]$ as in Lemma \ref{theorem: 2 layer Du}, if $n_1 \geq N$, then every critical point $\overline{\mathcal W}$ of $\phi_r(\mathbf x; \mathcal W)$ that is non-degenerate on $\{W_2, b_2\}$ is a global minimum of $\phi_r$. 
\end{theorem}
The following remark allows us to make our final conclusion from Theorem \ref{theorem: width N}.
\begin{remark}
As explained in \cite{a54a6932-695b-3fed-b1c3-8acbcfb5005c} and \cite{nguyen2017loss}, for a function $f$ in $C^2$ that maps an open subset $U\subset \mathbb R^n$ to $\mathbb R$, the degenerate critical points in $U$ are rare as the set of all degenerate points has Lebesgue measure zero.
\end{remark}

Theorem \ref{theorem: width N} provides sufficient conditions for a global minimum of the residual loss with a wide network of width $N$ or higher. Based on this theorem, if a PINN with a width of at least $N$ has a non-degenerate critical point, then it also has a global minimum for the residual loss. Since the residual loss $\phi_r(\mathbf x; \cdot)$ is a function from $\mathbb R^{|\mathcal W|}$ to $\mathbb R$ and has continuous second derivatives, the degenerate critical points are rare, and the wide PINN in Theorem \ref{theorem: width N} has a global minimum.

Note that the residual loss is a strong regularizer that results in a data-efficient training process. Thus, PINNs are often trained with $\mathcal{O}(1000)$ collocation points and even fewer boundary data \cite{raissi2017physics,krishnapriyan2021characterizing}. Therefore, the constraint on the width is well within the practical settings of neural networks. Furthermore, as we observe in the experiments in Section \ref{sec: experiments}, while satisfying the constraint on the width improves the performance, one can expect relatively good results with smaller width as long as the other conditions in Assumptions \ref{assumption: 2 layer single var} are almost satisfied. 

\section{On the Choice of Activation Function}
\begin{figure}[t]
    \centering
    \begin{subfigure}[t]{0.238\textwidth}
        \centering
        \includegraphics[width=\textwidth]{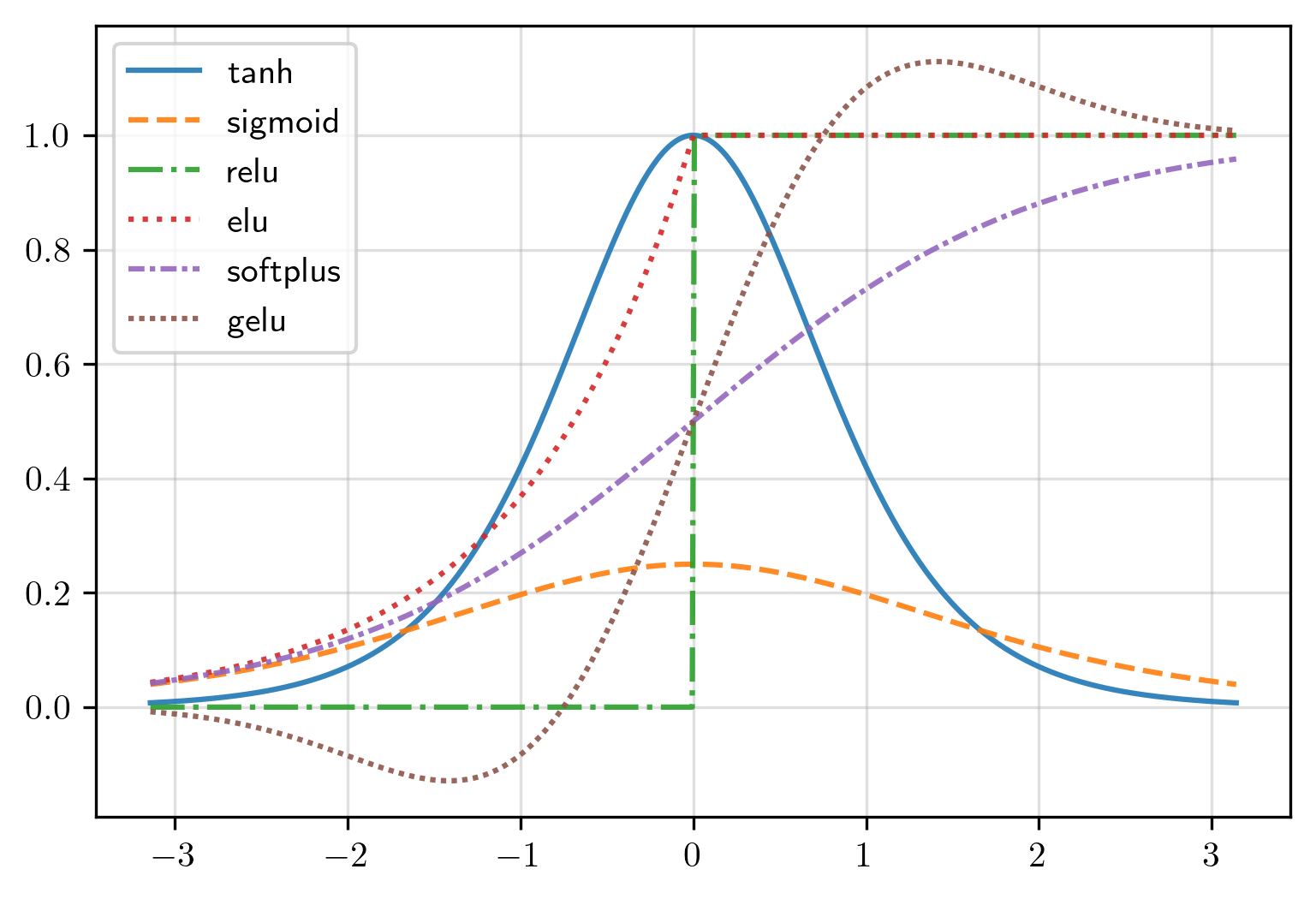}
        \caption{First derivatives}
    \end{subfigure}    
    \hfill
    \begin{subfigure}[t]{0.238\textwidth}
        \centering
        \includegraphics[width=\textwidth]{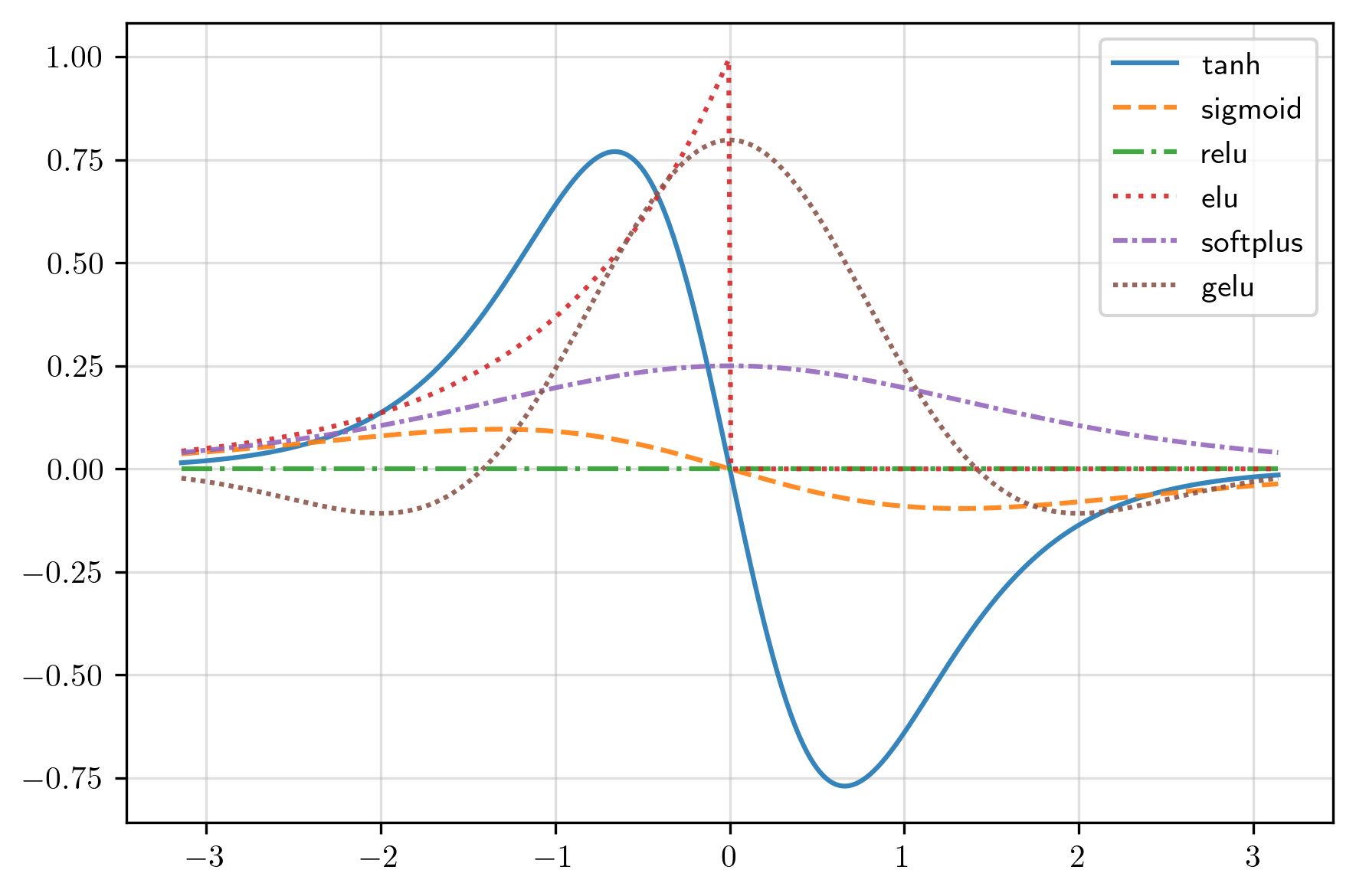}
        \caption{Second derivatives}
    \end{subfigure}    
    \caption{Derivatives of most of the common activation functions are not bijective. Here, only Softplus has a bijective first derivative.}
    \label{fig: derivatives of activations}
\end{figure}
\begin{figure}[t]
    \centering
    \begin{subfigure}[t]{0.238\textwidth}
        \centering
        \includegraphics[width=\textwidth]{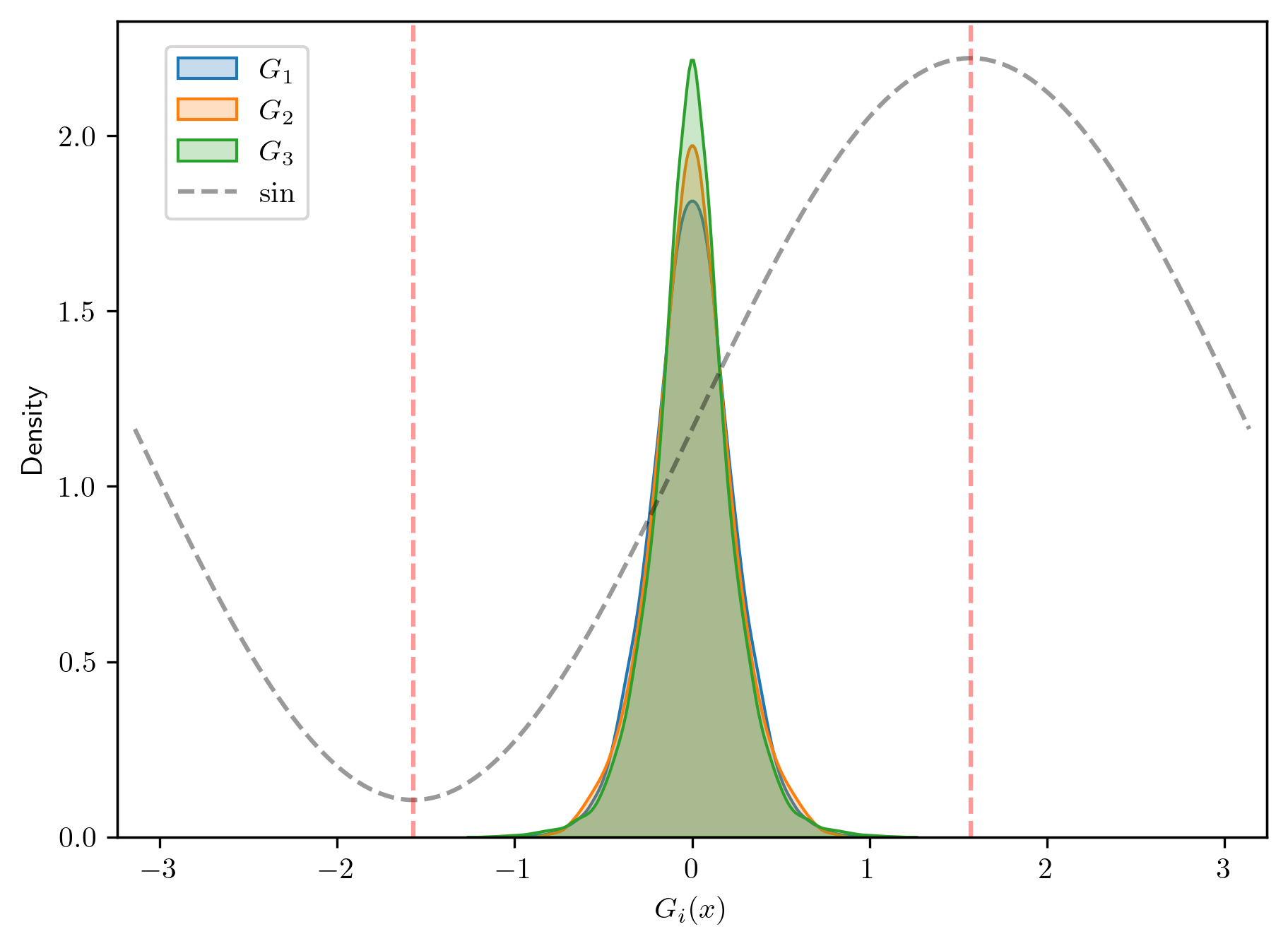}
        \caption{3-Layer 64-Neuron}
    \end{subfigure}
    \hfill
    \begin{subfigure}[t]{0.238\textwidth}
        \centering
        \includegraphics[width=\textwidth]{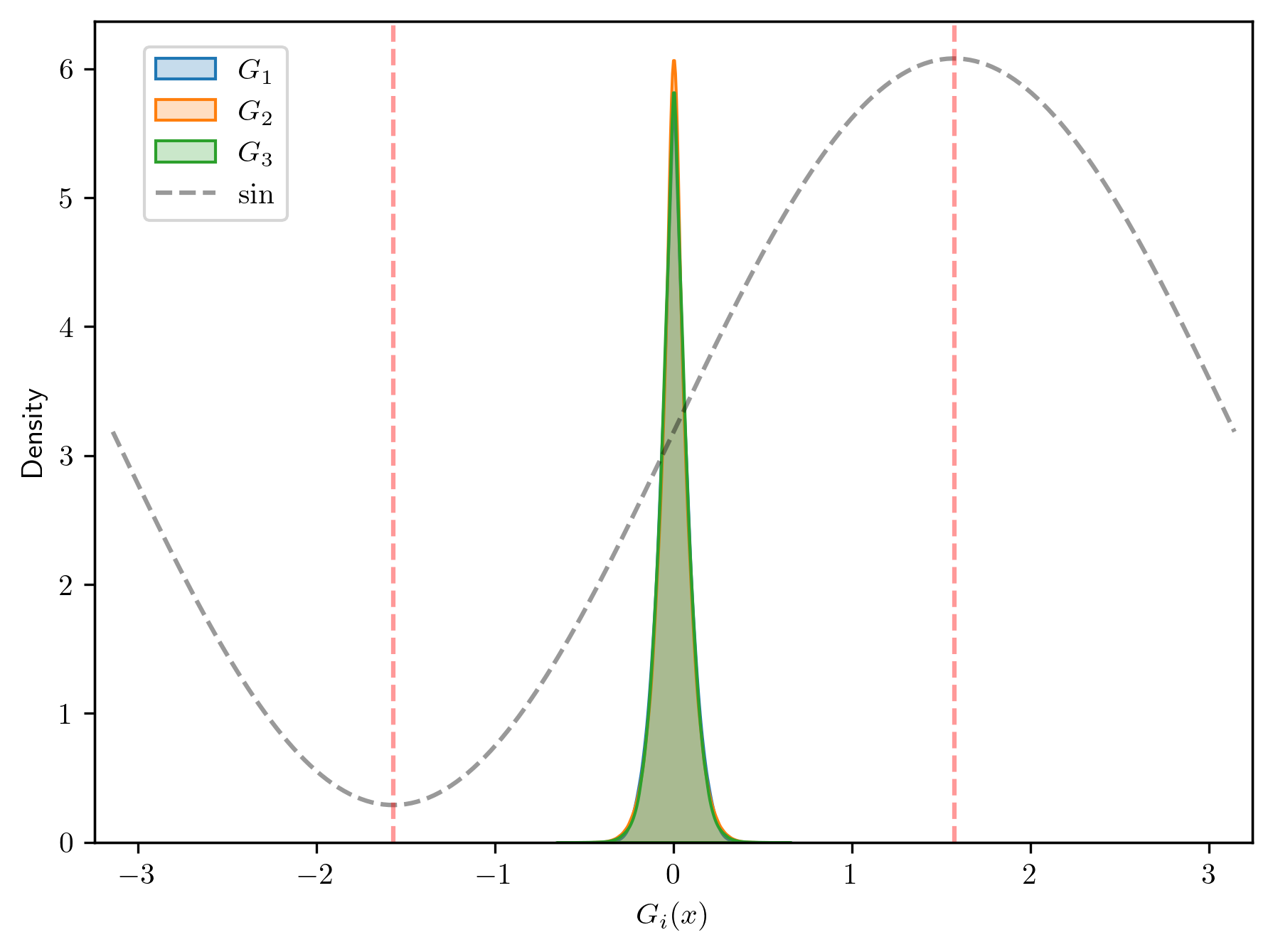}
        \caption{3-Layer 512-Neuron}
    \end{subfigure}
    \caption{Distribution of the linear outputs of the layers in Sine networks at initialization.}
    \label{fig: sine init}
\end{figure}

The conditions outlined in Theorem \ref{theorem: width N} and Assumptions \ref{assumption: 2 layer single var} collectively establish an important set of necessities for achieving global minimization of the residual loss. Notably, the requirement of strictly monotonically increasing $\sigma^{(k)}$ implies that it should be a bijection, providing an important guideline in choosing effective activation functions for PINNs. 
It is noteworthy that bijective activation functions are widely prevalent in deep learning, and extending this characteristic to their derivatives for improved expressiveness in representing differential operators is a plausible goal.

However, the activation functions frequently used in deep learning do not satisfy the bijection property even for the first-order derivatives. As depicted in Figure \ref{fig: derivatives of activations}, only Softplus has a bijective first-order derivative, and as we show in Section \ref{sec: experiments}, it indeed improves the performance of the first-order Transport PINN significantly. Meanwhile, there has been an increasing interest in the use of sinusoidal functions either as feature embeddings \cite{wong2022learning,wang2021eigenvector} or activation functions \cite{sitzmann2020implicit,belbute2022simple} for PINNs.

As shown in Figure \ref{fig: sine init}, we observe that the linear outputs of the layers in a neural network with the Sine activation function at initialization are centred at zero with low variance when initialized with normal Xavier initialization. Sitzmann et al. \shortcite{sitzmann2020implicit} also proposed a uniform initialization scheme for Sine networks that produces normal linear outputs at all layers with a desired variance. Consequently, in both cases, most of the linear outputs of the layers lie in the $\left[-\pi/2, \pi/2\right]$ interval where Sine is bijective. Furthermore, as we train the PINNs with the Sine activation function, we observe that layers still exhibit the same behaviour, i.e., most of the linear outputs of the layers are between $-\pi/2$ and $\pi/2$ after convergence, especially as the width grows larger. Figure \ref{fig: layer outputs} illustrates the output distributions for each layer of the trained Wave and Klein-Gordon PINNs (We later define these equations in Section \ref{sec: experiments}). 

The observations above suggest that the sinusoidal functions can be utilized to almost satisfy the bijective condition of the activation function. Specifically, we use Cosine and Sine activation functions to train PINNs with first- and second-order terms, respectively. As a result, as long as the width is adequately large to produce low-variance pre-activations within $\left[-\pi/2, \pi/2\right]$, the first-order terms in Cosine networks and the second-order terms in Sine networks are determined with the bijective interval of Sine. The same approach can be taken when solving PDEs with higher odd or even terms. 

The experiments in the next section show that sinusoidal non-linearity greatly improves the performance of PINNs compared to the common Tanh activation, and the gains are often greater as the width grows. We note that while the Assumption \ref*{assumption: 2 layer single var}.\ref{assumption: infimum} facilitates the proof of the Theorem \ref{theorem: width N}, the crucial property is the bijective $\sigma^{(k)}$, and we relax the assumption on the infimum of the derivatives.

\begin{figure}[t]
    \centering
    \begin{subfigure}[t]{0.23\textwidth}
        \centering
        \includegraphics[width=\textwidth]{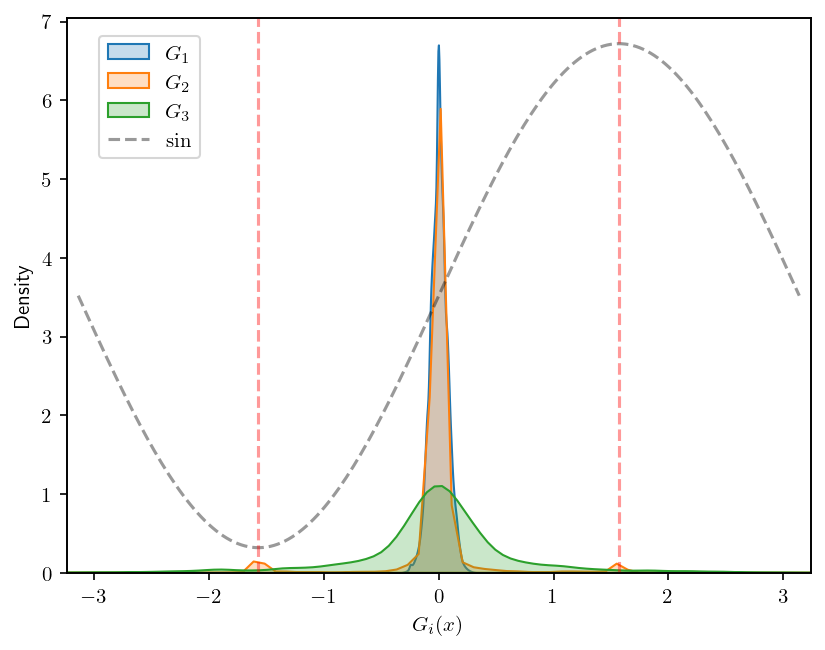}
    \end{subfigure}
    \hfill
    \begin{subfigure}[t]{0.23\textwidth}
        \centering
        \includegraphics[width=\textwidth]{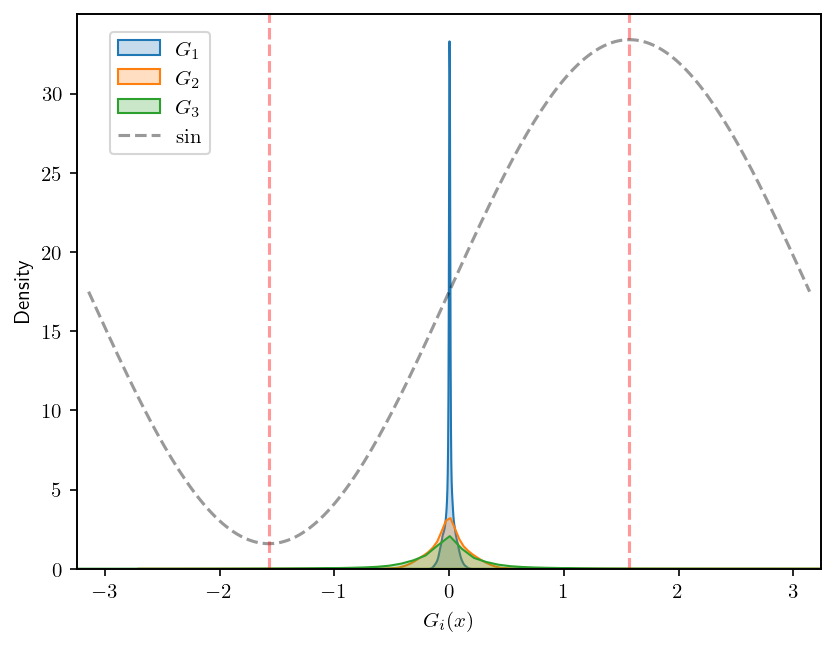}
    \end{subfigure}

    \begin{subfigure}[t]{0.23\textwidth}
        \centering
        \includegraphics[width=\textwidth]{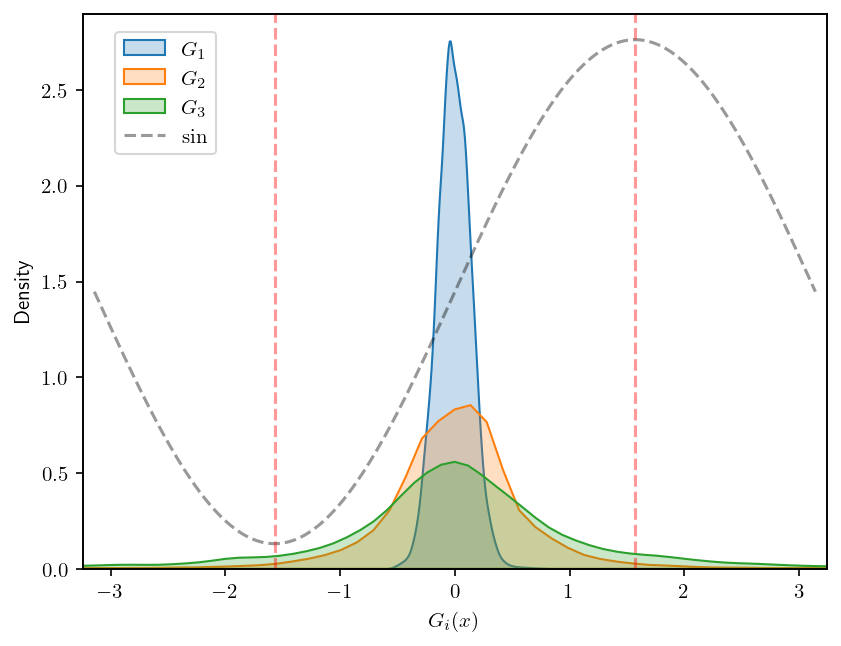}
        \caption{Wave Equation}
    \end{subfigure}
    \hfill
    \begin{subfigure}[t]{0.23\textwidth}
        \centering
        \includegraphics[width=\textwidth]{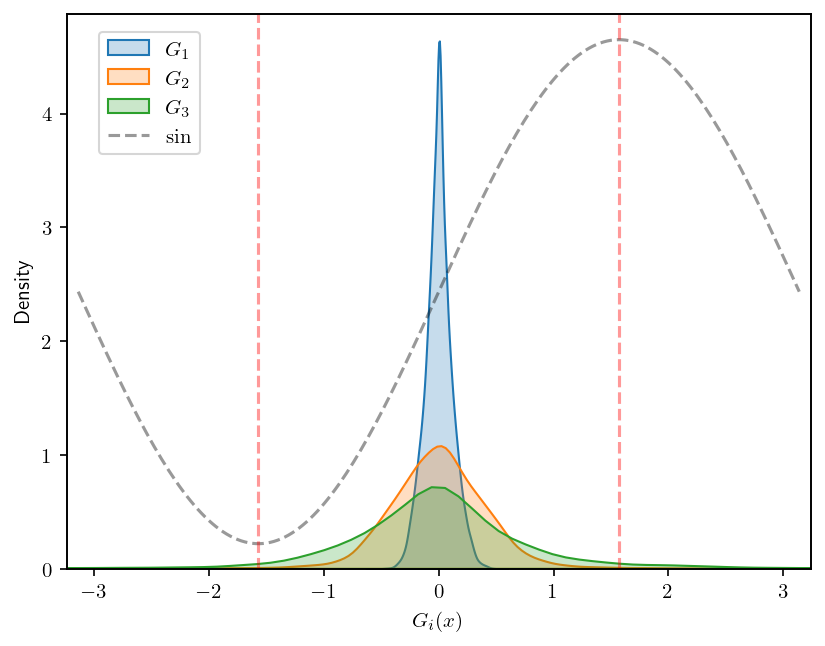}
        \caption{Klein-Gordon Equation}
    \end{subfigure}  
    \caption{Distribution of linear outputs of PINNs' layers. Top row: 1024 neurons wide, Bottom row: 256 neurons wide}
    \label{fig: layer outputs}
\end{figure}

\section{Experiments}\label{sec: experiments}

In this section, we provide numerical results for several PDEs, revealing the impact of the activation functions and the width. We first experiment with the first-order Transport equation, comparing the Softplus and Cosine activation functions with Tanh. Next, we study three second-order PDEs using Sine and Tanh activation functions. We empirically show that Sine significantly improves the performance of PDEs with second-order terms with a noticeable decrease in error as the width exceeds the number of training samples. 

In all of the experiments, we use a three-layer feed-forward network with a width varying from 64 neurons up to 1024 neurons and initialized with Normal Xavier initialization. All the models are trained with normalized inputs for 80K epochs using the Adam optimizer and an exponential learning rate decay scheme. The only exception is the Wave equation, for which the models are trained for 120K epochs for better convergence. We repeat each experiment three times with a different random initialization and report the average and the best results.

\subsection{Transport Equation}\label{sec: transport}

\begin{figure}[t]
    \centering
    \begin{subfigure}[t]{0.32\linewidth}
        \centering
        \includegraphics[width=\textwidth]{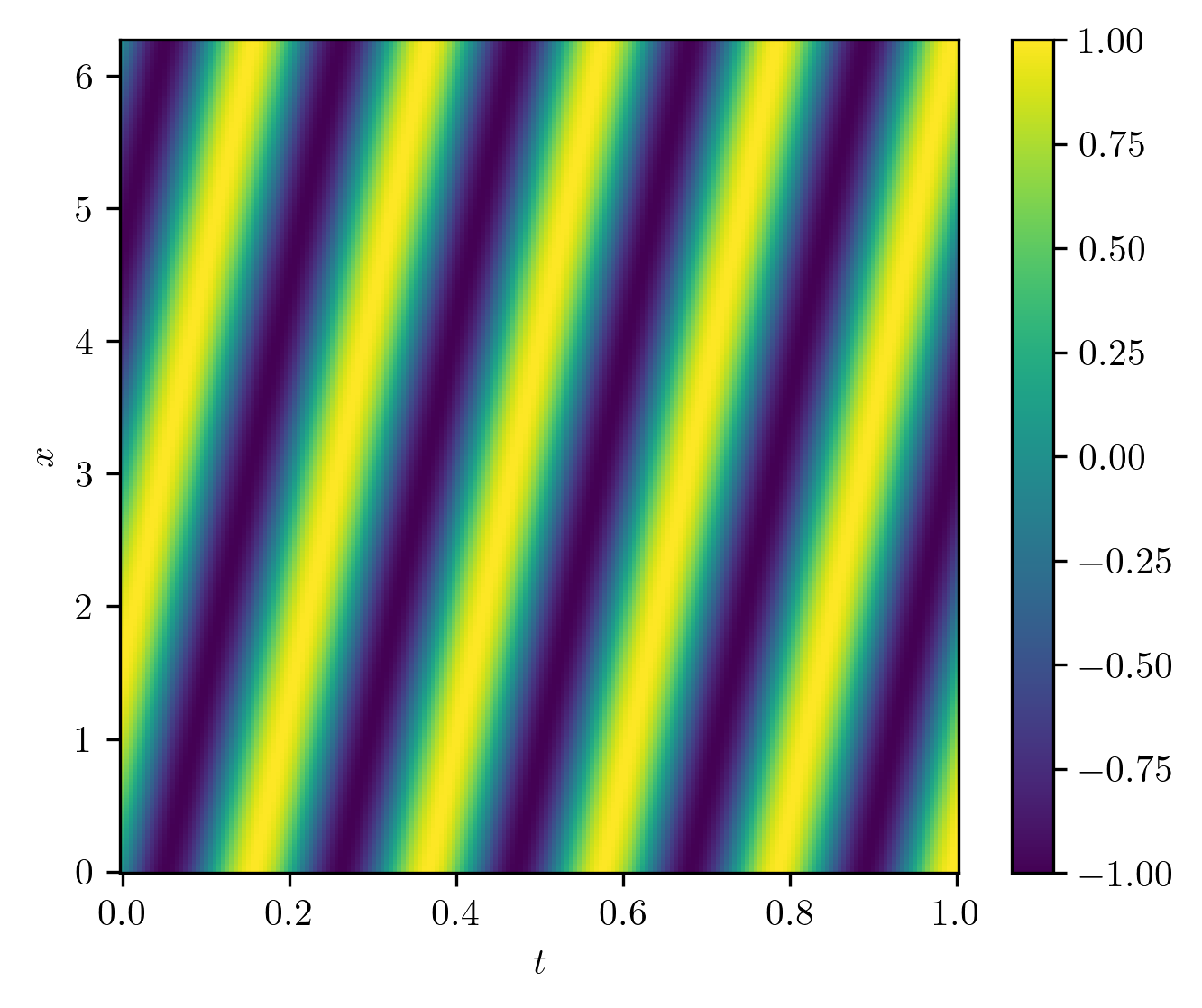}
    \end{subfigure}    
    \hfill
    \begin{subfigure}[t]{0.32\linewidth}
        \centering
        \includegraphics[width=\textwidth]{figs/transport_exact.png}
    \end{subfigure}
    \hfill
    \begin{subfigure}[t]{0.32\linewidth}
        \centering
        \includegraphics[width=\textwidth]{figs/transport_exact.png}
    \end{subfigure} 
    
    \begin{subfigure}[t]{0.32\linewidth}
        \centering
        \includegraphics[width=\textwidth]{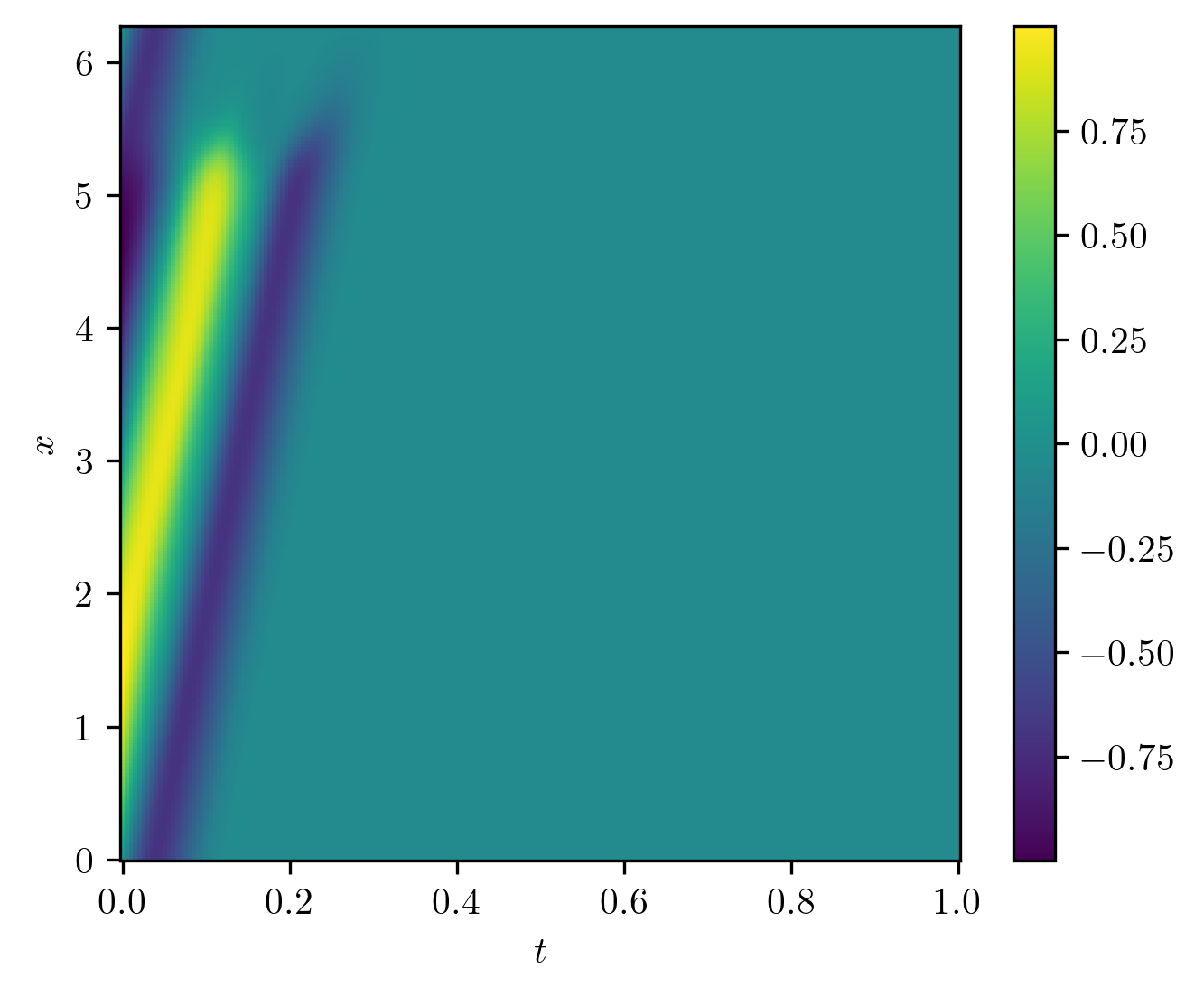}
    \end{subfigure}    
    \hfill
    \begin{subfigure}[t]{0.32\linewidth}
        \centering
        \includegraphics[width=\textwidth]{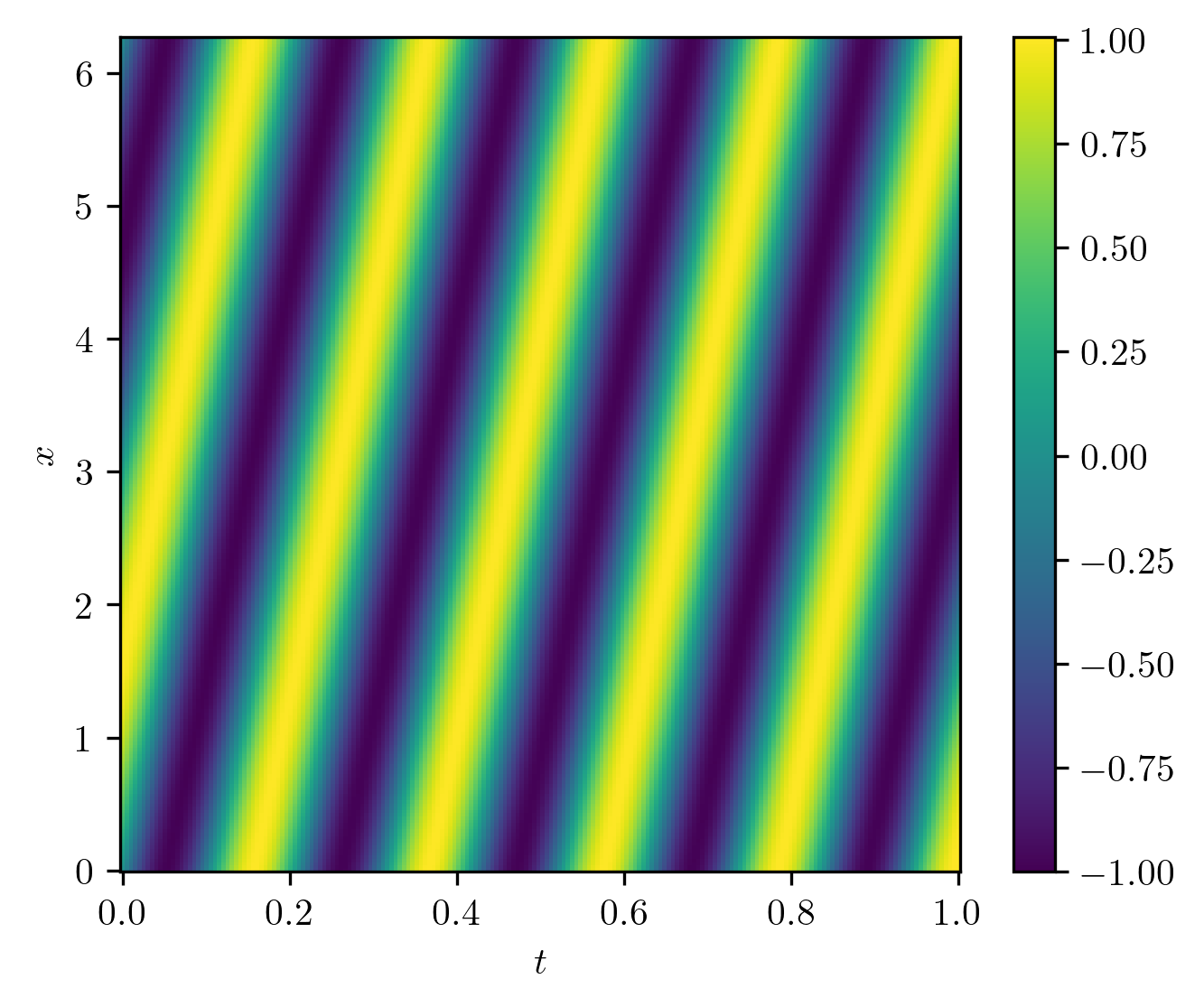}
    \end{subfigure}   
    \hfill
    \begin{subfigure}[t]{0.32\linewidth}
        \centering
        \includegraphics[width=\textwidth]{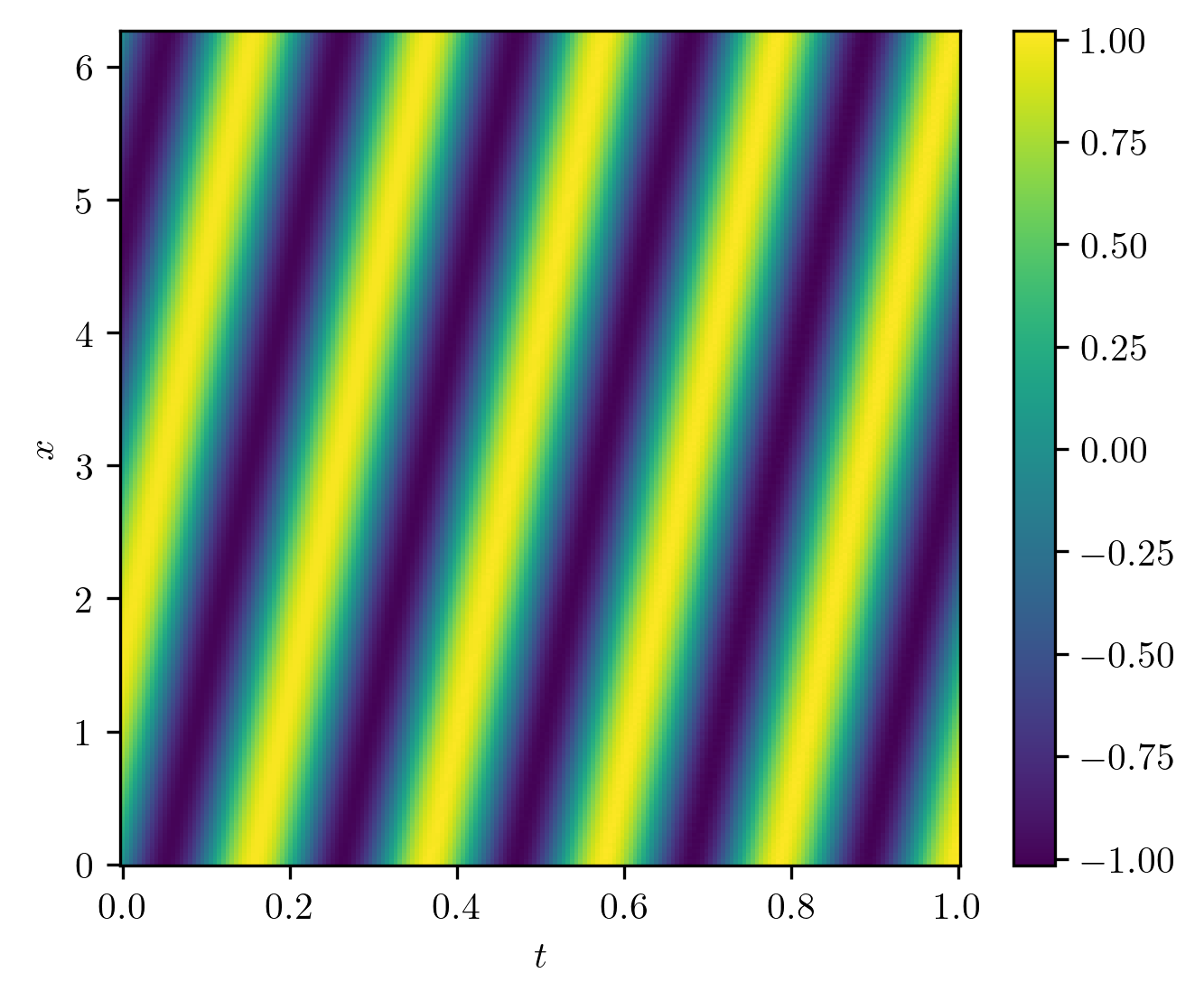}
    \end{subfigure} 

    \begin{subfigure}[t]{0.32\linewidth}
        \centering
        \includegraphics[width=\textwidth]{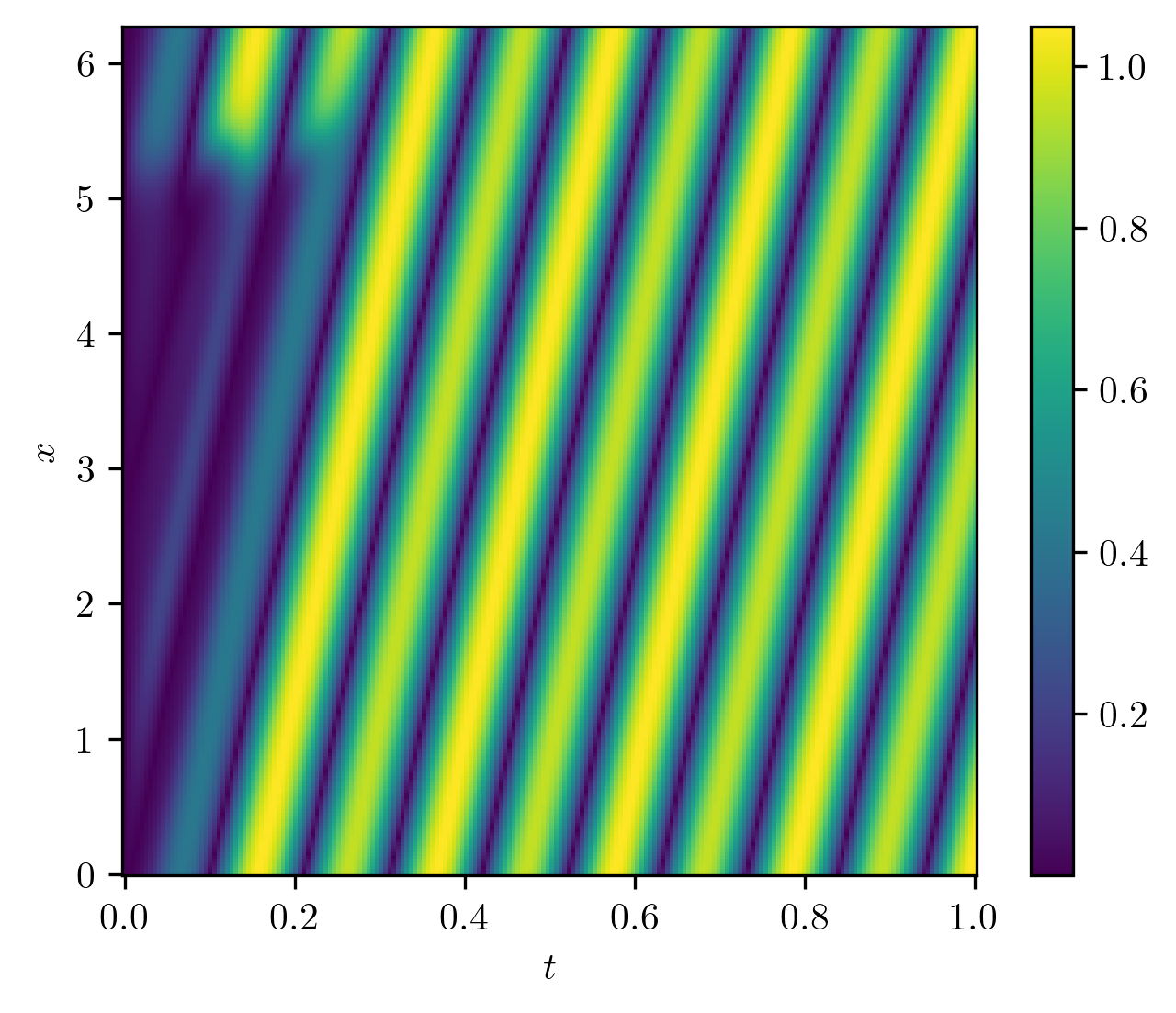}
        \caption{Tanh 256-Neuron}
    \end{subfigure}    
    \hfill
    \begin{subfigure}[t]{0.32\linewidth}
        \centering
        \includegraphics[width=\textwidth]{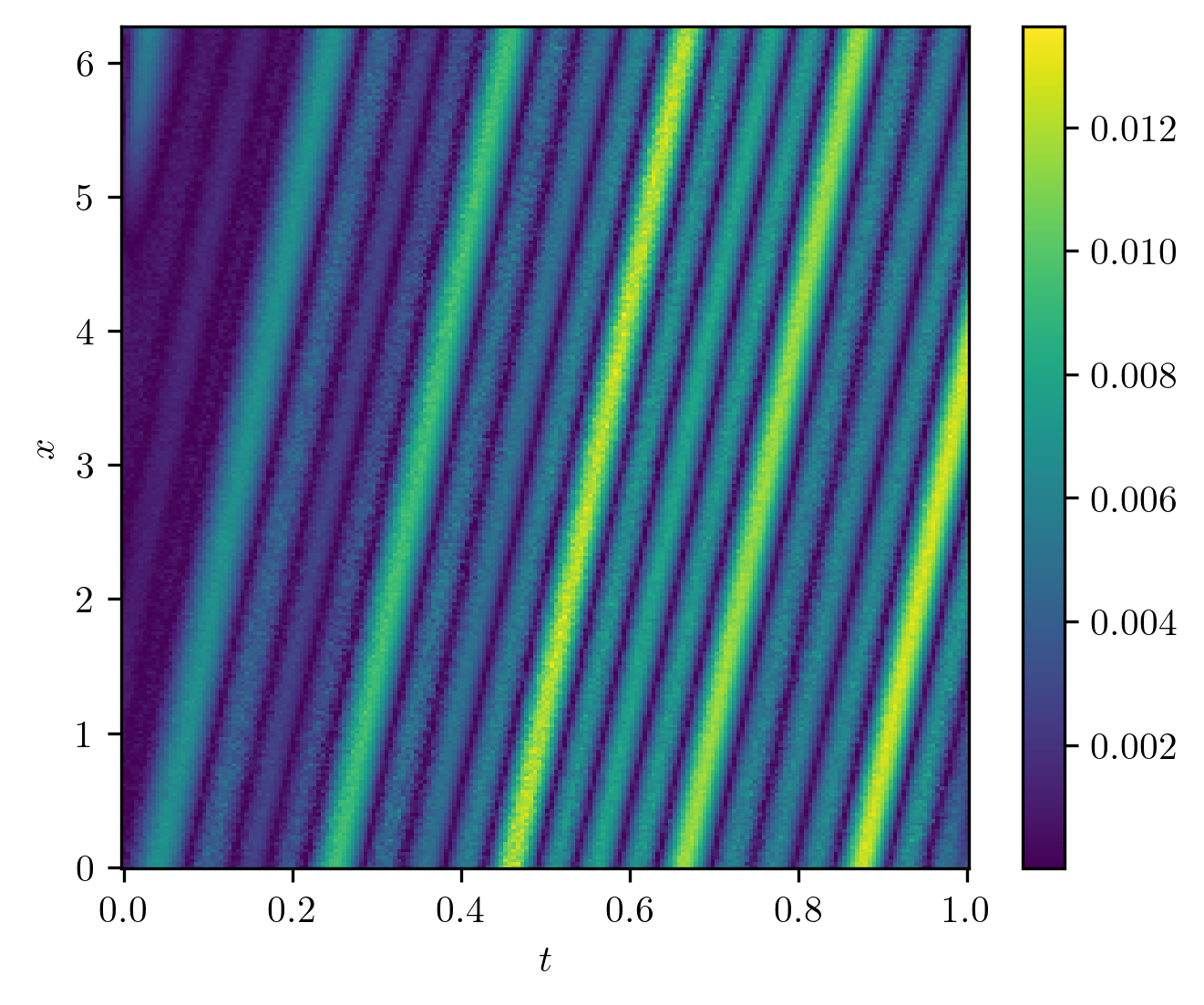}
        \caption{Cosine 256-Neuron}
    \end{subfigure}    
    \hfill
    \begin{subfigure}[t]{0.32\linewidth}
        \centering
        \includegraphics[width=\textwidth]{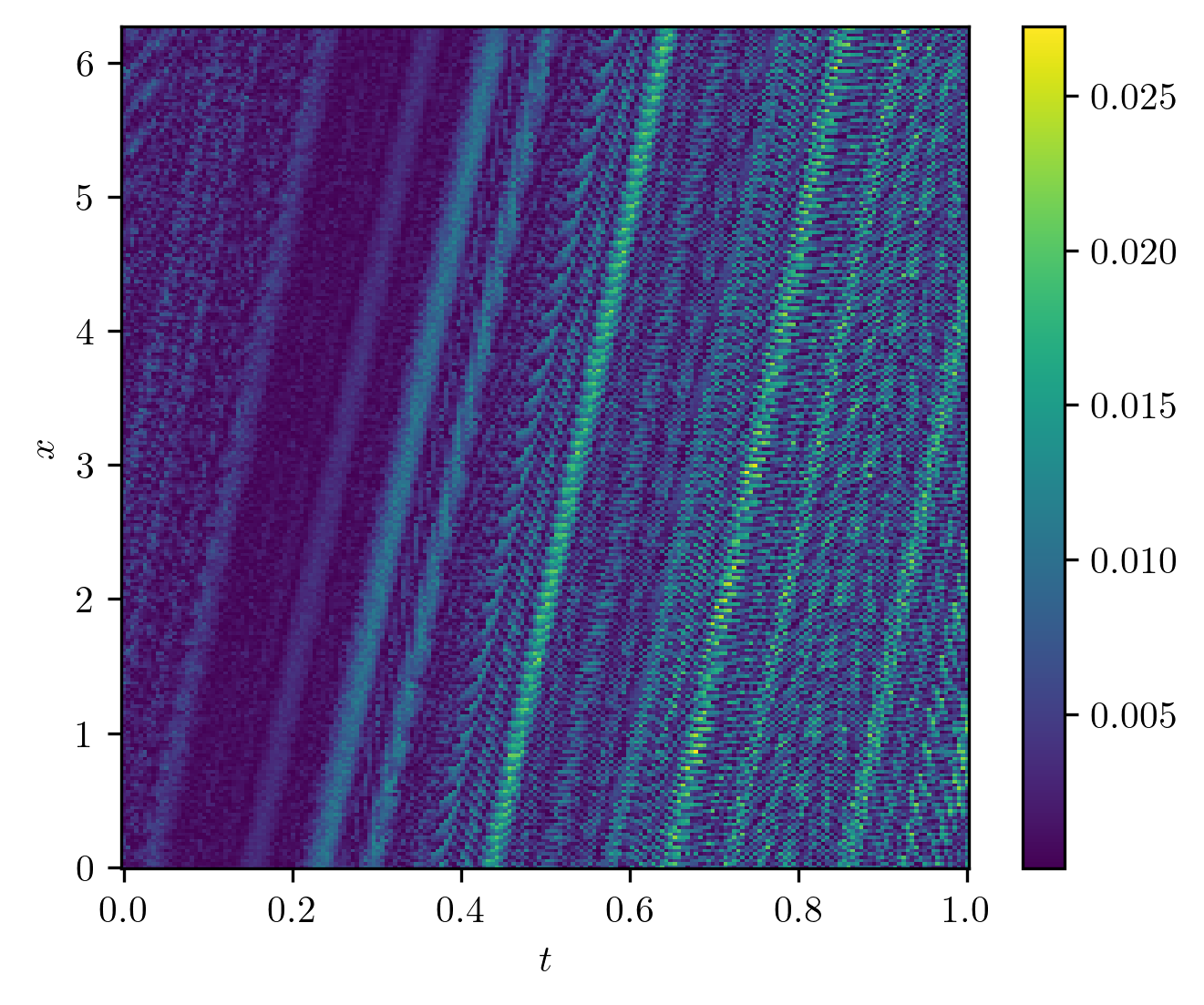}
        \caption{Softplus 256-Neuron}
    \end{subfigure} 

    \caption{Transport Equation. Top panels: Exact solution, Middle panels: Predicted solution, Bottom panels: Absolute Error}
    \label{fig: trasnport preds}
\end{figure}

The transport equation is a first-order linear PDE that describes a quantity as it moves through time and space. We experiment with the one-dimensional equation with the following formulation:
\begin{equation*}
        \frac{\partial u}{\partial t} + 30 \frac{\partial u}{\partial x} = 0 \;\;\;\; x\in\left[0, 2\pi\right], t\in\left[0, 1\right]
\end{equation*}
We also impose a periodic boundary condition $u(t, 0) = u(t, 2\pi)$ and a Dirichlet initial condition $u(0, x)$ consistent with the solution in \cite{krishnapriyan2021characterizing}. The PINN is then trained with 256 collocation training points and 200 boundary samples. To verify our results in Section \ref{sec: width N and activation}, we choose Softplus as the activation function as it is a smooth version of ReLU with a bijective first derivative equal to the Sigmoid function. We also use Cosine, as it results in a Sine network for the first-order terms. 

Table \ref{tab: transport mae} reports the mean absolute errors for the transport equation. In all cases, Softplus and Cosine perform significantly better than Tanh, as shown in Figure \ref{fig: trasnport preds}. Furthermore, as the width becomes equal to the number of collocation samples (256), we observe a noticeable decrease in the absolute error. The same improvement is also evident in the training curve of the residual loss as shown in Figure \ref{fig: residual curve transport softplus}, where the wide models follow a steep curve. For the PINNs with a width of 256 or wider with both Cosine and Softplus activation functions, the absolute errors are between $10^{-2}$ and $10^{-3}$, outperforming the reported $1.1\times 10^{-2}$ absolute error achieved with curriculum learning in \cite{krishnapriyan2021characterizing}. 

\begin{figure}[t]
    \centering
    \includegraphics[width=0.8\linewidth]{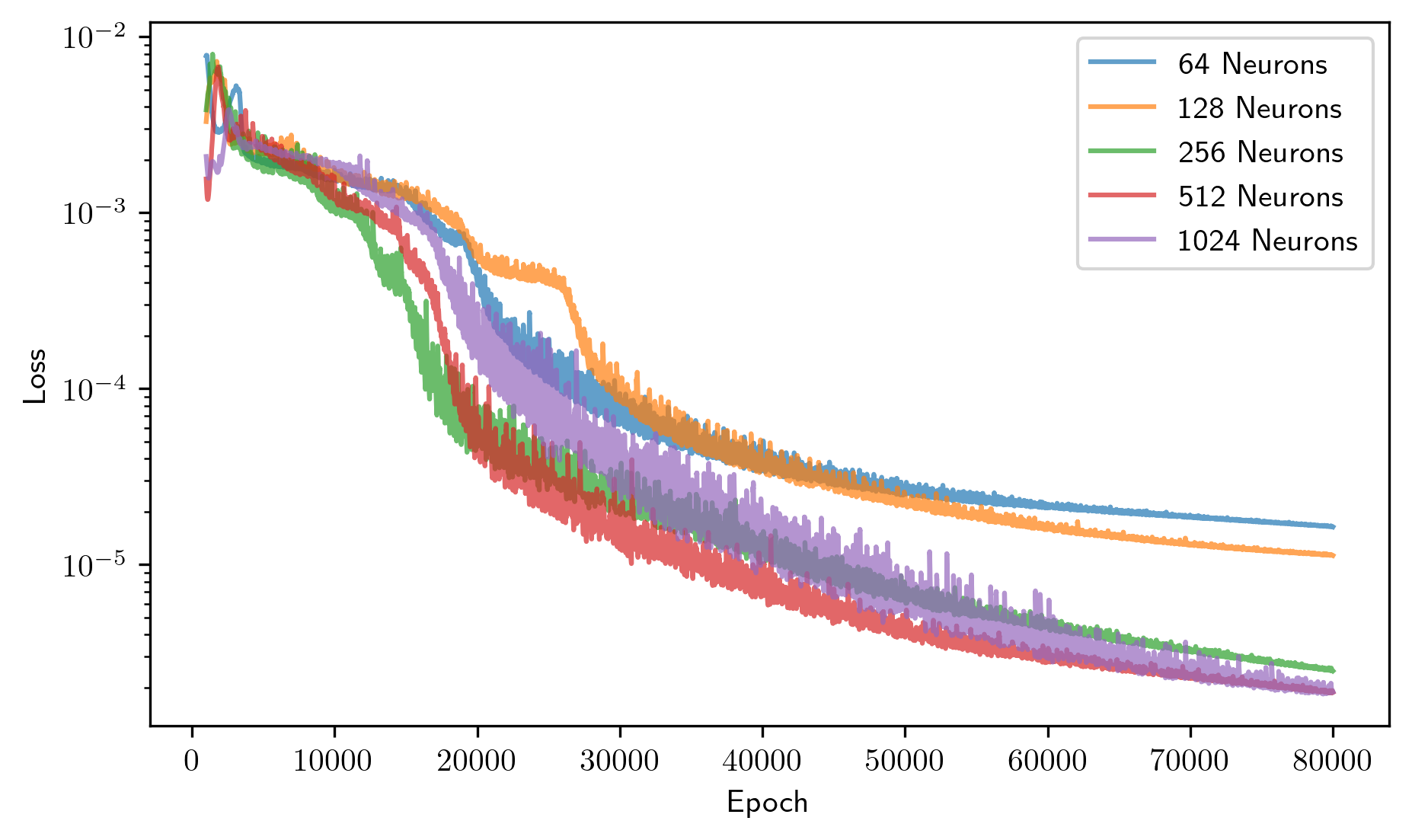}
    \caption{Average residual loss curve for the Transport PINNs with the Softplus activation function and trained with 256 collocation samples. }
    \label{fig: residual curve transport softplus}
\end{figure}

\begin{table}[t]
\centering
\resizebox{\linewidth}{!}{
\begin{tabular}{@{}c|rr|rr|rr@{}}
\toprule
      & \multicolumn{2}{c|}{Tanh} & \multicolumn{2}{c|}{Softplus} & \multicolumn{2}{c}{Cosine} \\ 
\multicolumn{1}{c|}{Width}           & \multicolumn{1}{c}{Avg}        & \multicolumn{1}{c|}{Best}        & \multicolumn{1}{c}{Avg}        & \multicolumn{1}{c|}{Best}       & \multicolumn{1}{c}{Avg}       & \multicolumn{1}{c}{Best}  \\ \midrule
64              & 0.6314       & 0.6212       & 0.0207      & 0.0111     & \textbf{0.0095}     & 0.0086      \\
128             & 0.6118       & 0.6083       & 0.0229      & 0.0156     & \textbf{0.0042}     & 0.0039      \\
\underline{256} & 0.5673       & 0.5588       & 0.0092      & 0.0040     & \textbf{0.0029}     & 0.0013      \\
512             & 0.5168       & 0.4158       & 0.0093      & 0.0062     & \textbf{0.0019}     & 0.0004      \\
1024            & 0.3632       & 0.0011       & 0.0103      & 0.0062     & \textbf{0.0014}     & 0.0011      \\ \bottomrule
\end{tabular}
}
\caption{Average and best mean absolute error for the Transport equation over three random initializations trained with 256 collocation points.}
\label{tab: transport mae}
\end{table}

\begin{table*}[t]
\centering
\resizebox{\textwidth}{!}{
\begin{tabular}{@{}crrrrrrrrrrrr@{}}
\toprule
\multicolumn{1}{l}{}       & \multicolumn{4}{c}{Helmholtz}                                                                              & \multicolumn{4}{c}{Klein Gordon}                                                                                  & \multicolumn{4}{c}{Wave}                                                                                \\ \midrule
\multicolumn{1}{l|}{}      & \multicolumn{2}{c}{Tanh}                           & \multicolumn{2}{c|}{Sine}                             & \multicolumn{2}{c}{Tanh}                           & \multicolumn{2}{c|}{Sine}                             & \multicolumn{2}{c}{Tanh}                           & \multicolumn{2}{c}{Sine}                           \\
\multicolumn{1}{c|}{Width} & \multicolumn{1}{c}{Avg} & \multicolumn{1}{c}{Best} & \multicolumn{1}{c}{Avg} & \multicolumn{1}{c|}{Best}   & \multicolumn{1}{c}{Avg} & \multicolumn{1}{c}{Best} & \multicolumn{1}{c}{Avg} & \multicolumn{1}{c|}{Best}   & \multicolumn{1}{c}{Avg} & \multicolumn{1}{c}{Best} & \multicolumn{1}{c}{Avg} & \multicolumn{1}{c}{Best} \\ \midrule
\multicolumn{1}{c|}{64}    & 4.7235                  & 4.1456                   & \textbf{0.0125}         & \multicolumn{1}{r|}{0.0087} & 0.0275                  & 0.0167                   & \textbf{0.0097}         & \multicolumn{1}{r|}{0.0059} & 0.3166                  & 0.2676                   & \textbf{0.1966}         & 0.1272                   \\
\multicolumn{1}{c|}{128}   & 2.8161                  & 2.2007                   & \textbf{0.0105}         & \multicolumn{1}{r|}{0.0064} & {0.0134}            & 0.0096                   & \textbf{0.0056}         & \multicolumn{1}{r|}{0.0035} & 0.2902                  & 0.2351                   & \textbf{0.1377}         & 0.0594                   \\
\multicolumn{1}{c|}{256}   & 1.8516                  & 1.4016                   & \textbf{0.0273}         & \multicolumn{1}{r|}{0.0212} & \underline{0.0606}                  & \underline{0.0191}                   & {\underline{\textbf{0.0052}}}   & \multicolumn{1}{r|}{\underline{0.0031}} & \textbf{0.1937}         & 0.1808                   & 0.3608                  & 0.0676                   \\
\multicolumn{1}{c|}{512}   & \underline{0.7854}                  & \underline{0.4909}                   & {\underline{\textbf{0.0044}}}   & \multicolumn{1}{r|}{\underline{0.0028}} & 0.0875                  & 0.0289                   & \textbf{0.0096}         & \multicolumn{1}{r|}{0.0049} & \underline{0.1577}                  & \underline{0.0725}                   & {\underline{\textbf{0.0604}}}   & \underline{0.0587}                   \\
\multicolumn{1}{c|}{1024}  & {0.5946}            & 0.2221                   & \textbf{0.0067}         & \multicolumn{1}{r|}{0.0056} & 0.1939                  & 0.0928                   & \textbf{0.0090}         & \multicolumn{1}{r|}{0.0046} & {0.1147}            & 0.0311                  & \textbf{0.0620}         & 0.0562                   \\ \bottomrule
\end{tabular}
}
\caption{Average and best mean absolute errors for second-order PDEs over three random initializations. Underlined values show where the width is equal to $N$.}
\label{tab: table second order}
\end{table*}

\begin{figure}[ht]
    \centering
    \begin{subfigure}[t]{0.32\linewidth}
        \centering
        \includegraphics[width=\textwidth]{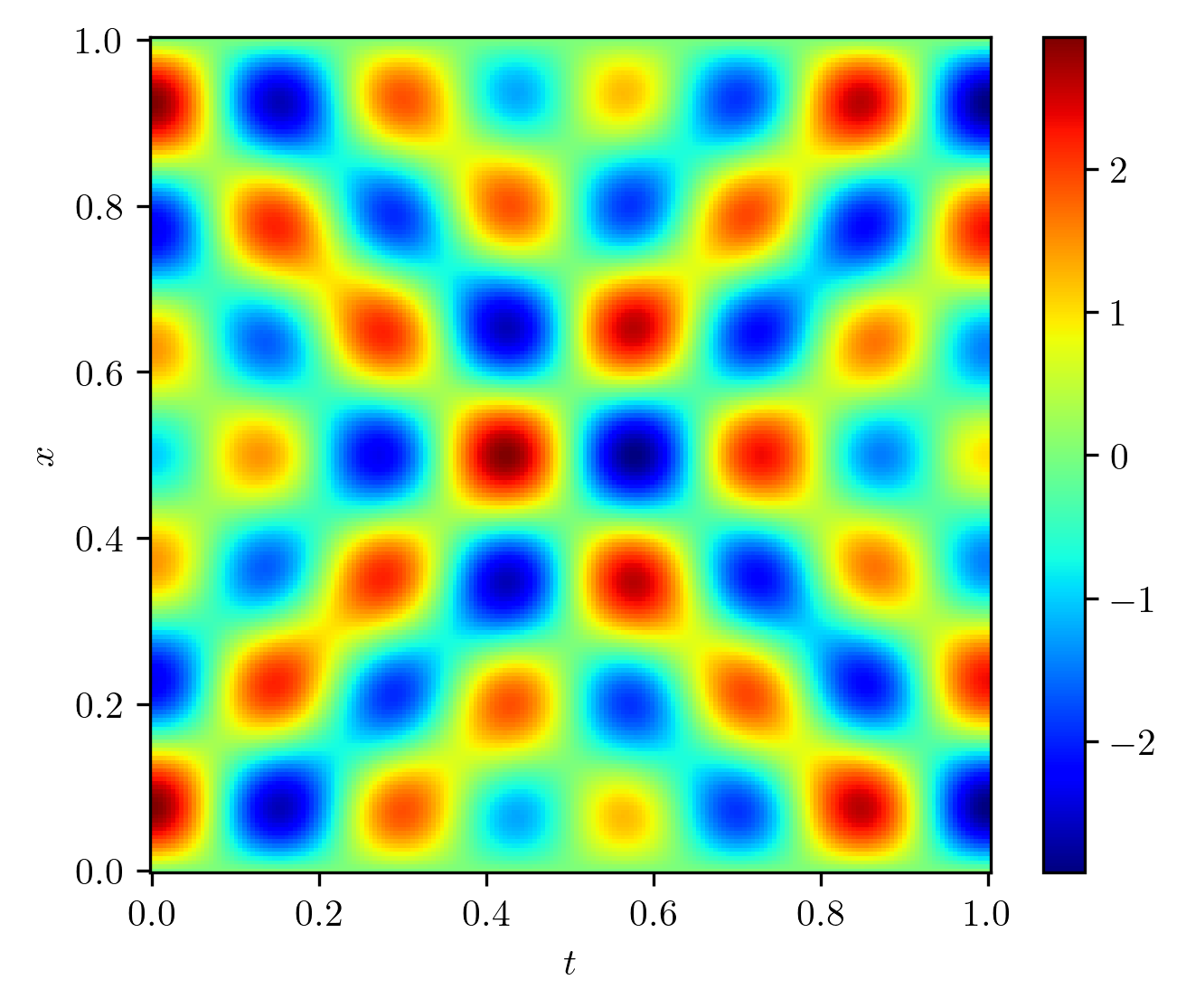}
    \end{subfigure}    
    \hfill
    \begin{subfigure}[t]{0.32\linewidth}
        \centering
        \includegraphics[width=\textwidth]{figs/wave_exact.png}
    \end{subfigure}
    \hfill
    \begin{subfigure}[t]{0.32\linewidth}
        \centering
        \includegraphics[width=\textwidth]{figs/wave_exact.png}
    \end{subfigure} 
    
    \begin{subfigure}[t]{0.32\linewidth}
        \centering
        \includegraphics[width=\textwidth]{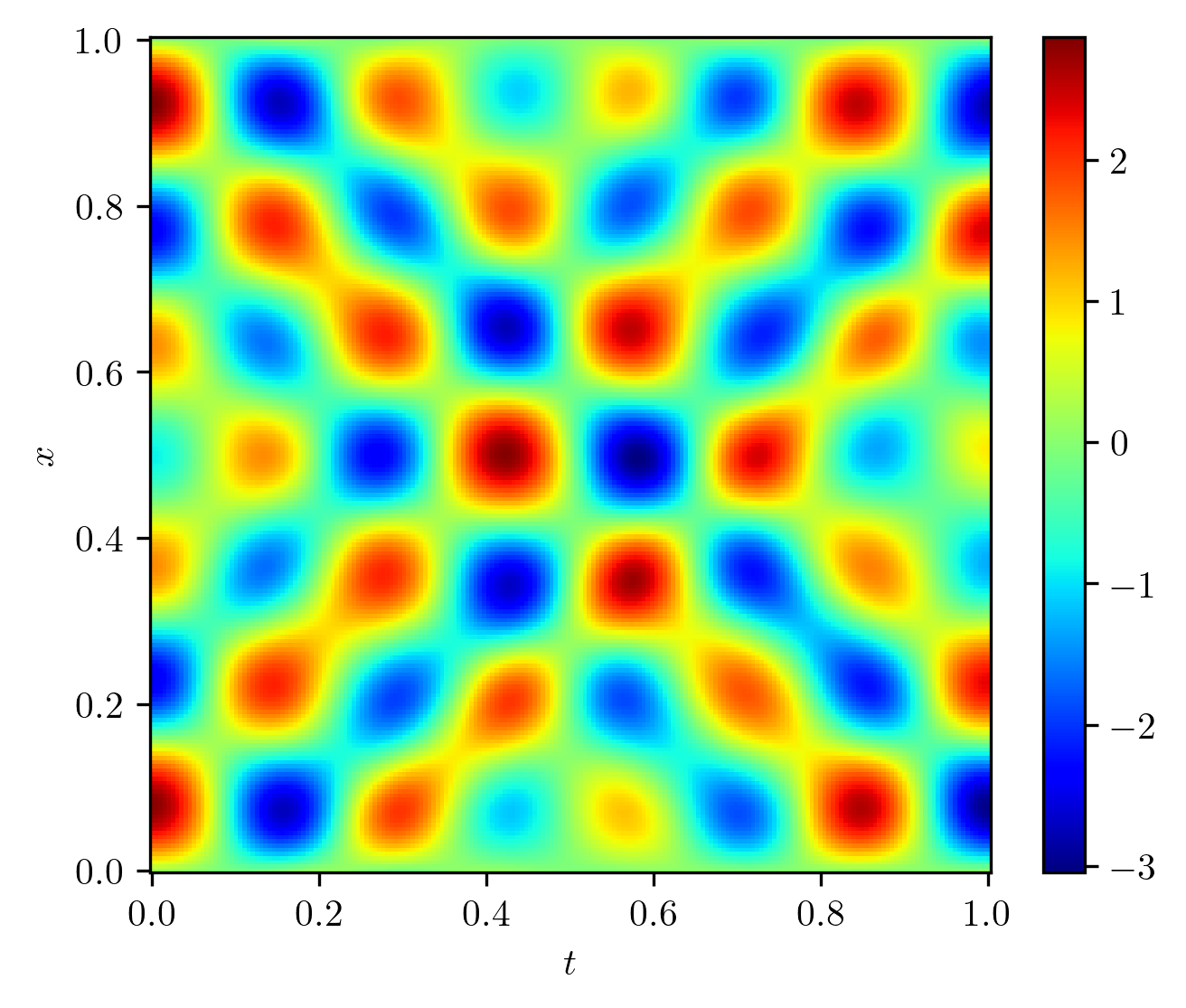}
    \end{subfigure}    
    \hfill
    \begin{subfigure}[t]{0.32\linewidth}
        \centering
        \includegraphics[width=\textwidth]{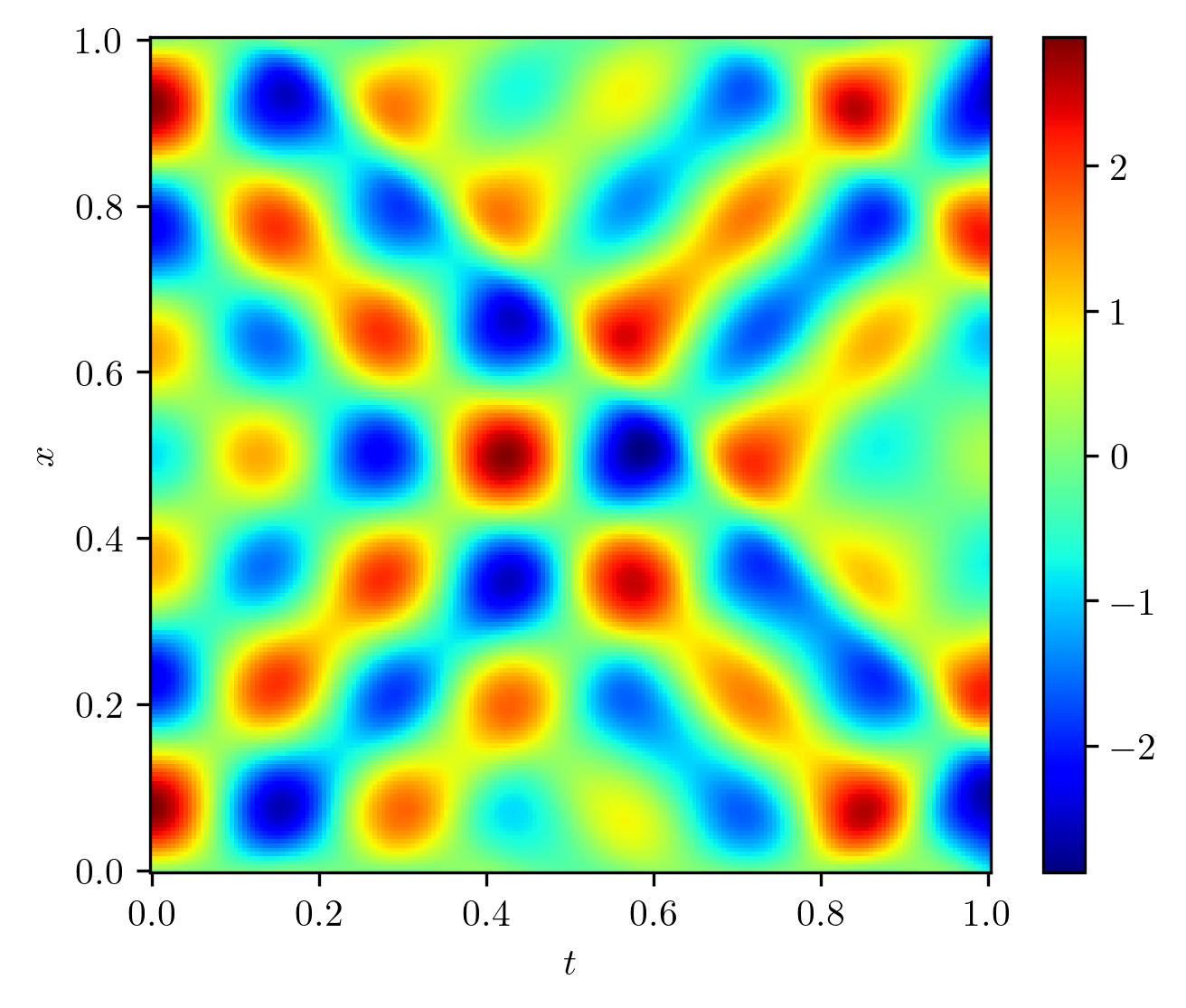}
    \end{subfigure}   
    \hfill
    \begin{subfigure}[t]{0.32\linewidth}
        \centering
        \includegraphics[width=\textwidth]{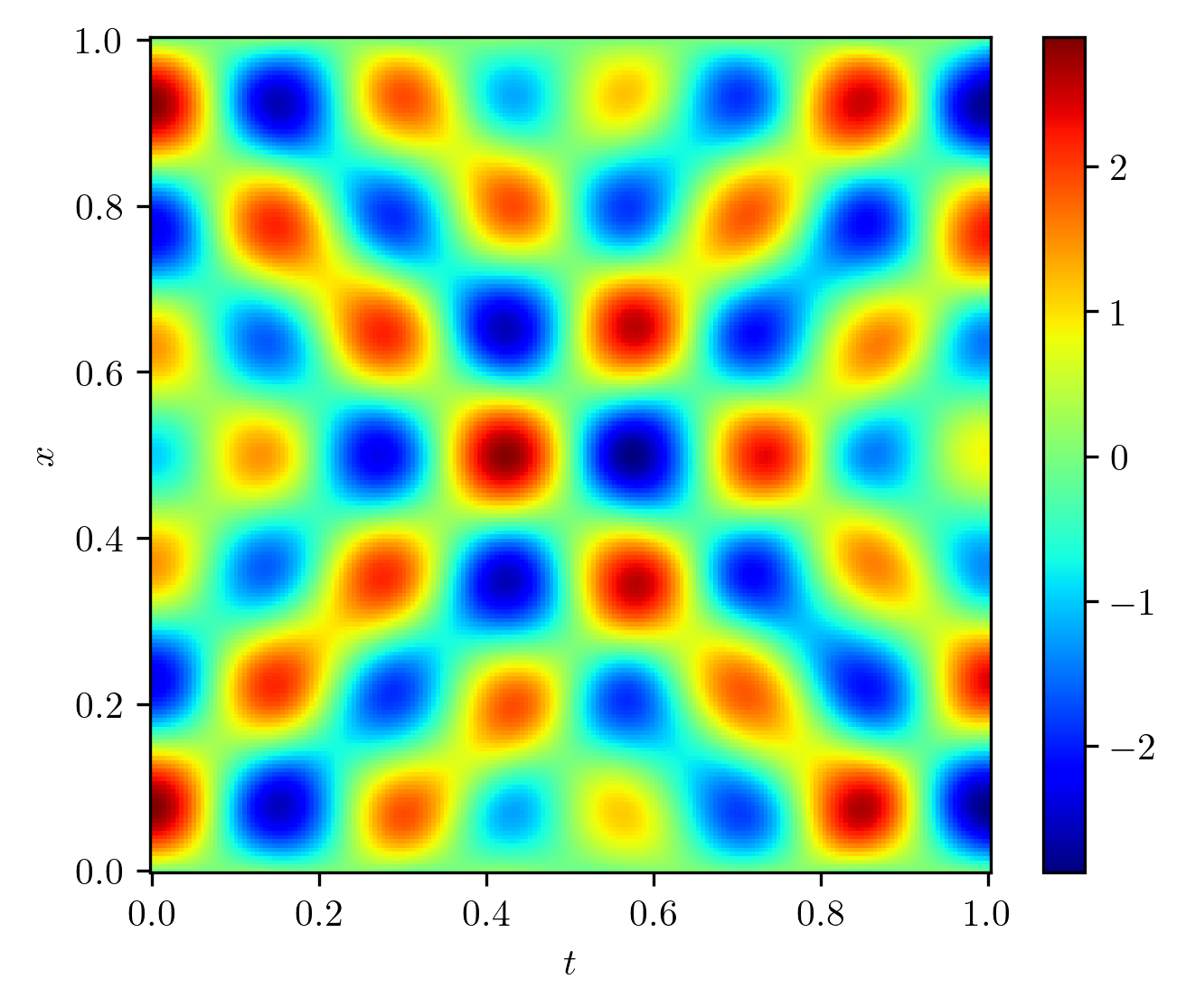}
    \end{subfigure} 

    \begin{subfigure}[t]{0.32\linewidth}
        \centering
        \includegraphics[width=\textwidth]{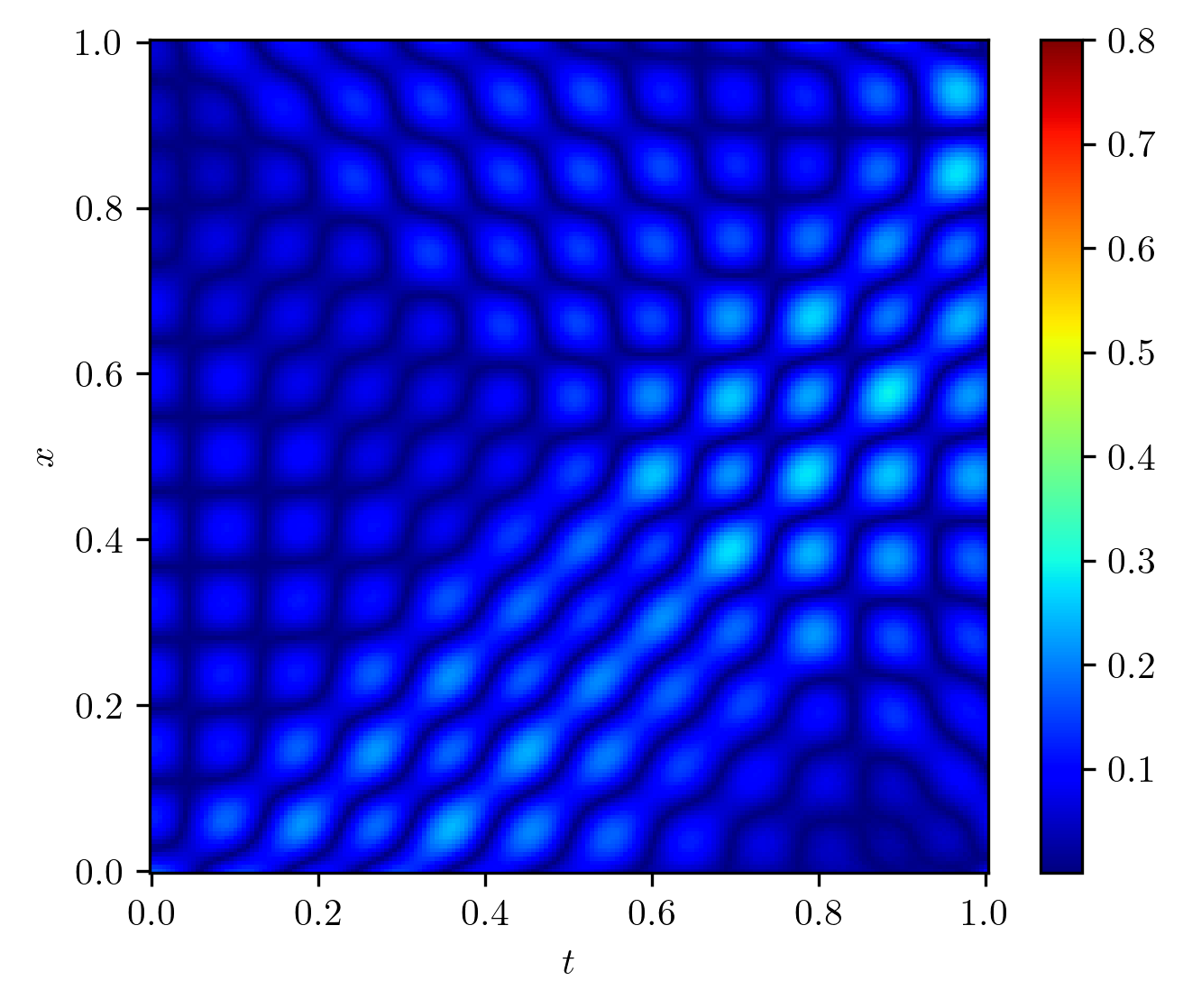}
        \caption{Sine 128-Neruon}
    \end{subfigure}    
    \hfill
    \begin{subfigure}[t]{0.32\linewidth}
        \centering
        \includegraphics[width=\textwidth]{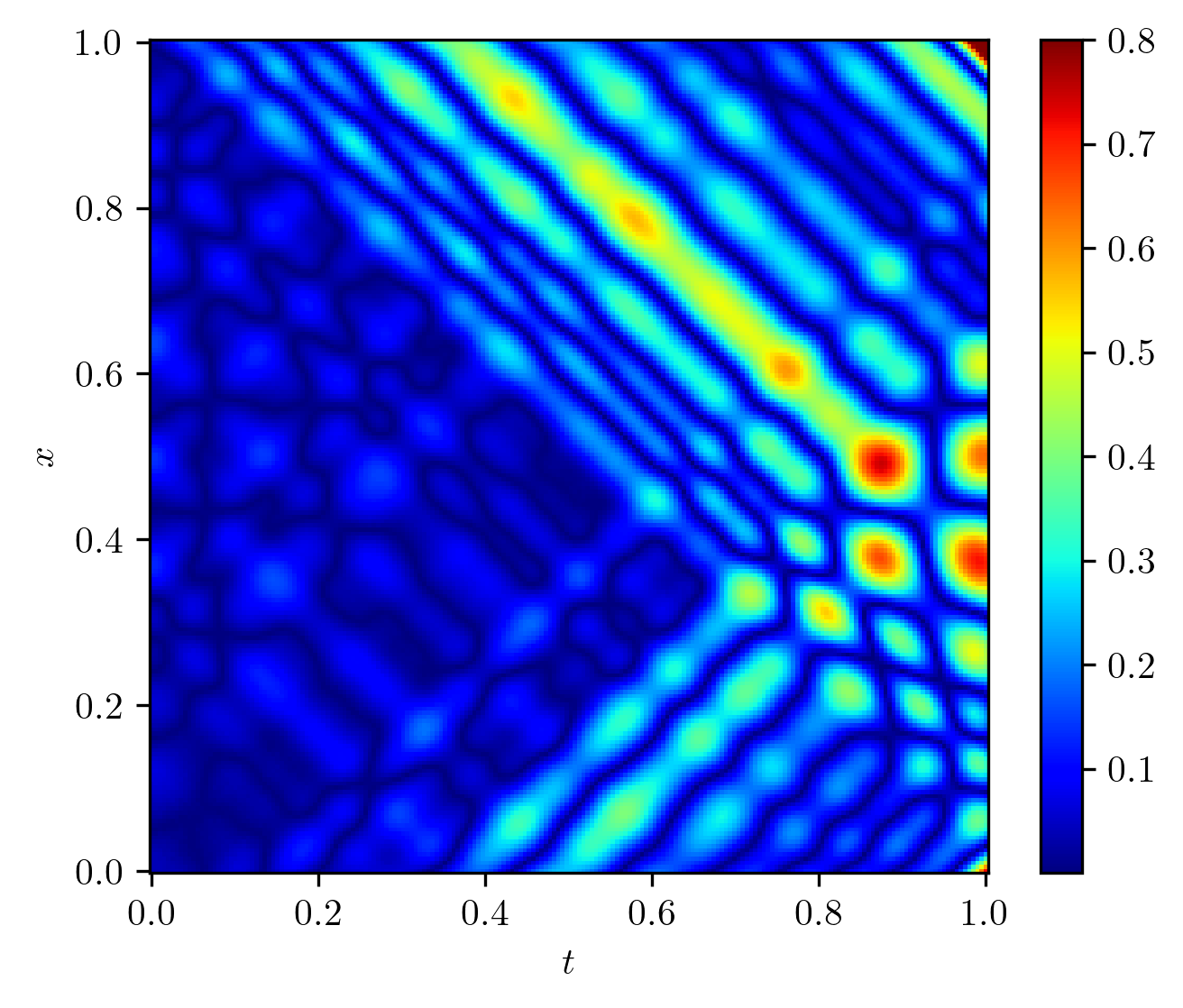}
        \caption{Tanh 128-Neuron}
    \end{subfigure}    
    \hfill
    \begin{subfigure}[t]{0.32\linewidth}
        \centering
        \includegraphics[width=\textwidth]{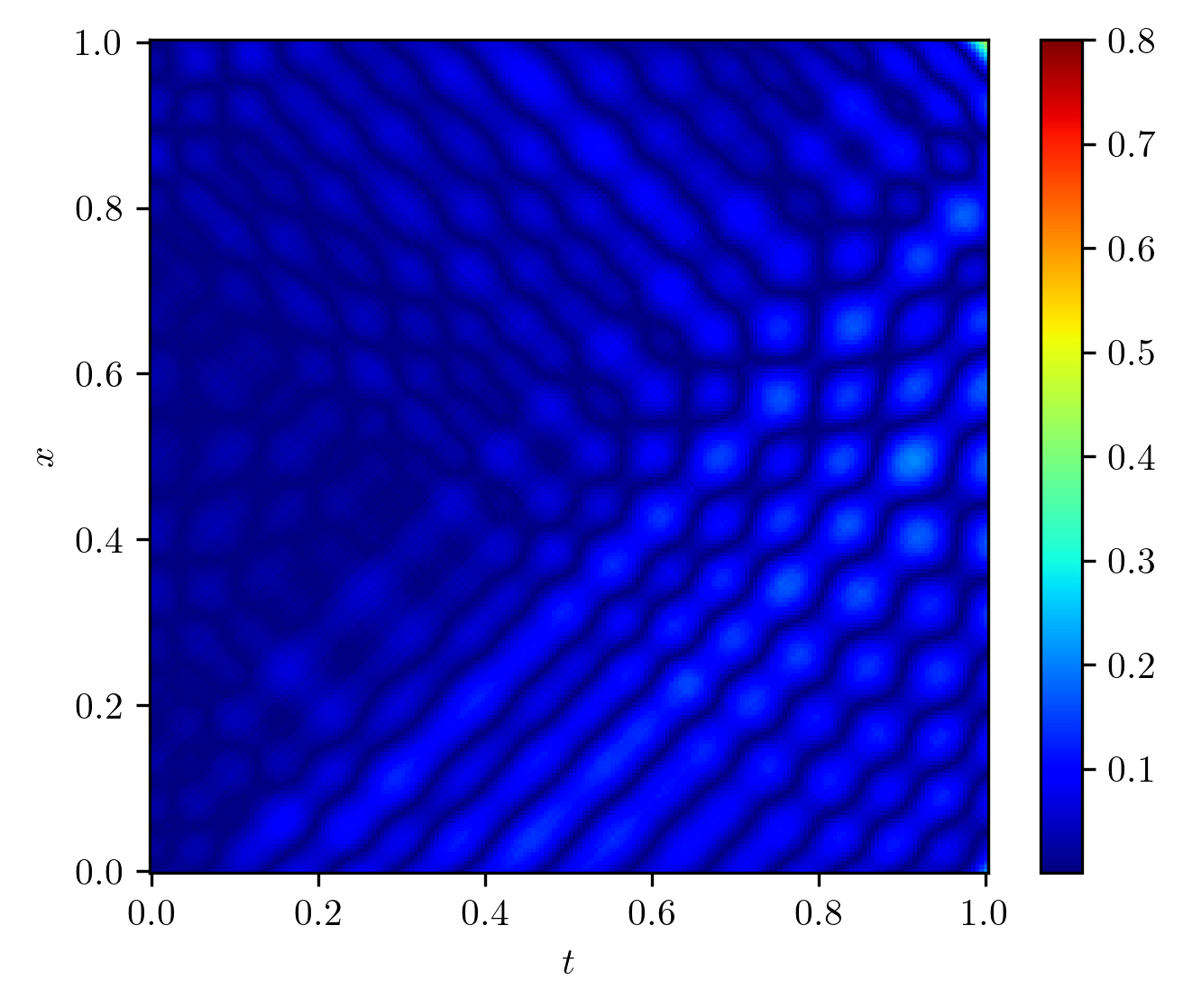}
        \caption{Tanh 1024-Neuron}
    \end{subfigure} 

    \caption{Wave Equation. Top panels: Exact solution, Middle panels: Predicted solution, Bottom panels: Absolute Error}
    \label{fig: wave preds}
\end{figure}

\begin{figure}[t]
    \centering
    \includegraphics[width=0.8\linewidth]{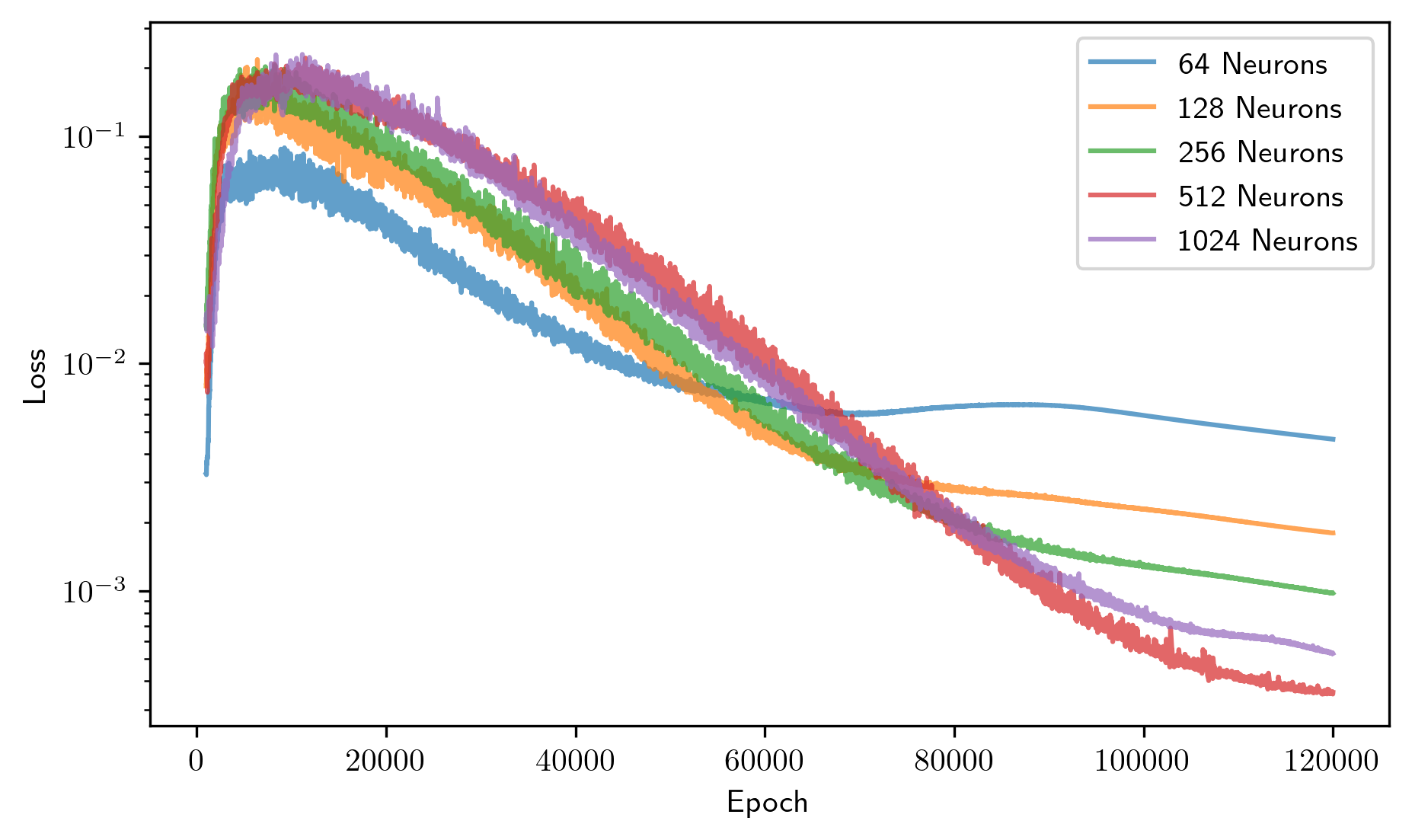}
    \caption{Average residual loss curve for Wave PINNs with the Tanh activation function, trained with 512 collocation samples.}
    \label{fig: residual curve wave}
\end{figure}

\subsection{Wave Equation}
The wave equation describes mechanical and electromagnetic waves and has the following form in 1-D: 
\begin{equation*}
    \frac{\partial^2 u}{\partial t^2} = c^2\frac{\partial^2 u}{\partial x ^ 2}
\end{equation*}
Here $c$ is the velocity of the wave. For $c=1$ and the solution
$$u(t, x) = \sin(5\pi x)\cos(5\pi t) + 2\sin(7\pi x)\cos(7 \pi t),$$
we train the PINN with 512 collocation points for $t, x\in[0, 1]$ and impose the initial and boundary conditions below with 256 boundary data points:
\begin{align*}
    &u(0, x) = \sin(5\pi x)+2\sin(7\pi x)\;\;& x \in \left[0, 1\right]\\
    &\frac{\partial u}{\partial t}(0, x) = 0 &x \in \left[0, 1\right] \\
    &u(t, 1) = u(t, 0) = 0 & t \in \left[0, 1\right]
\end{align*}

The residual loss training curves in Figure \ref{fig: residual curve wave} highlight the impact of the width in training PINNs, with wider models consistently achieving smaller loss values, and the 512- and 1024-neuron models following almost the same path. This behaviour is similar to the curves in Figure \ref{fig: residual curve transport softplus} for the transport equation, where all the models with a width of at least $N$ achieve very close loss values. As evident in Table \ref{tab: table second order}, PINNs with both Tanh and Sine activation functions perform notably better than narrow PINNs, with a mean absolute error of $3.11\times 10^{-2}$ and $5.62\times 10^{-2}$ for Tanh and Sine respectively. Figure \ref{fig: wave preds} illustrates the inability of the narrow Tanh network compared to Sine and wide Tanh models in representing the solution. Also, while narrow Sine PINNs are able to find good solutions, the training is more unstable and the performance is worse on average compared to wider models.

\subsection{Helmholtz Equation}
We consider a 2D Helmholtz equation of the following form with $t,x \in \left[-1, 1\right]$ as in \cite{wong2022learning}:
\begin{equation*}
    \begin{aligned}
        \frac{\partial^2 u}{\partial x^2} + \frac{\partial^2 u}{\partial y^2} + u = (1-\pi^2-(6\pi)^2)\sin(\pi x)\cos(6\pi y)
    \end{aligned}
\end{equation*}
With zero boundary conditions, the solution is given by
$$
u(x, y) = \sin(\pi x)\sin(6\pi y).
$$
The PINN is then trained with 512 collocation points and 256 boundary data, using Sine and Tanh as the activation function. As reported in Table \ref{tab: table second order}, Sine performs remarkably better than Tanh across all widths. Similar to the Wave and Transport equations, there is also a slight decrease in the errors as the width exceeds the number of collocation points. As in the Wave equation, PINNs with the Tanh activation function start to perform better as the width grows larger. However, as illustrated in Figure \ref{fig: Helmholtz preds}, even the 1024 neurons-wide Tanh network is still unable to capture the solution, while the Sine network finds an acceptable solution even with a width of 128 neurons.

\begin{figure}[t]
    \centering
    \begin{subfigure}[t]{0.32\linewidth}
        \centering
        \includegraphics[width=\textwidth]{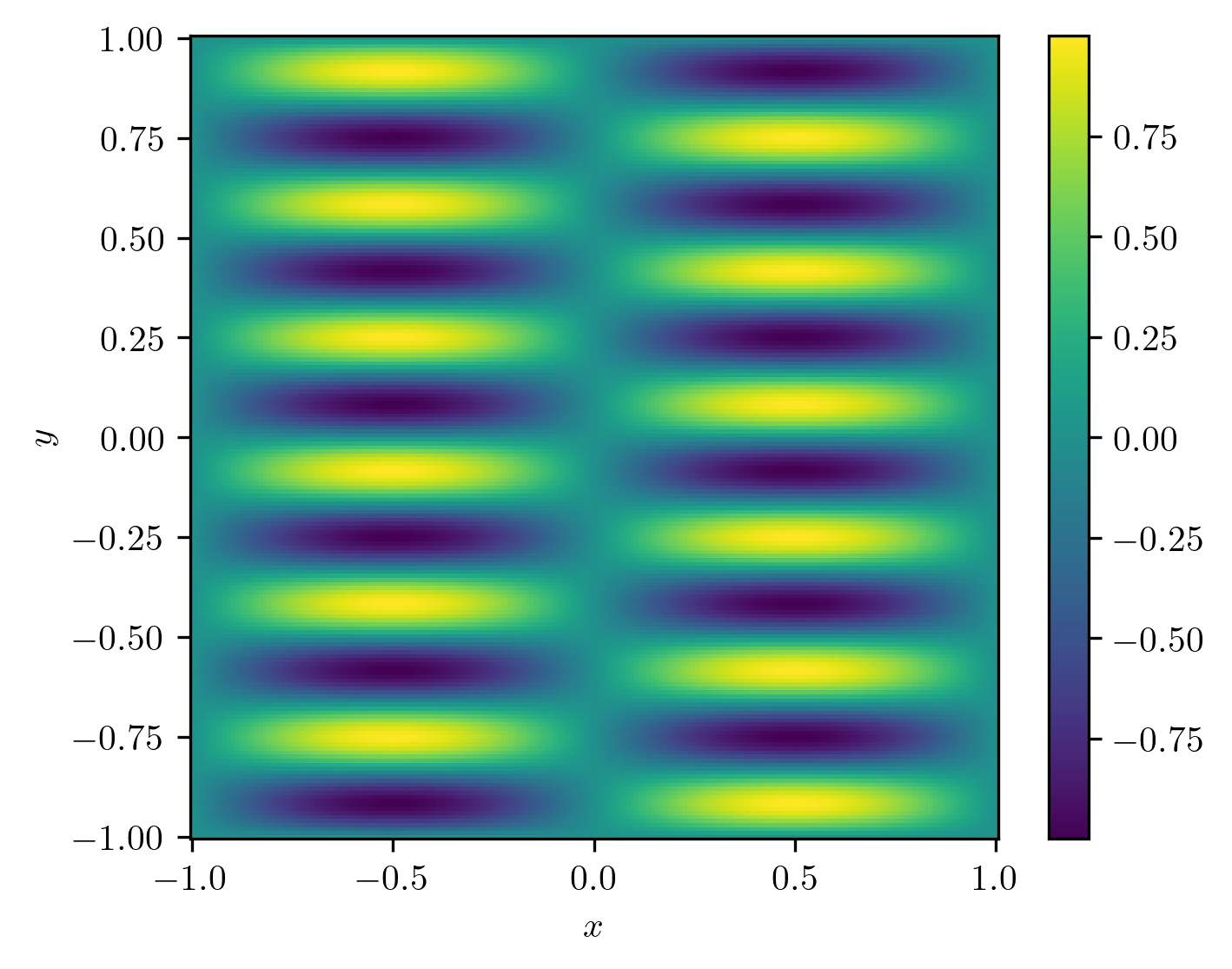}
    \end{subfigure}    
    \hfill
    \begin{subfigure}[t]{0.32\linewidth}
        \centering
        \includegraphics[width=\textwidth]{figs/helmholtz_exact.png}
    \end{subfigure}
    \hfill
    \begin{subfigure}[t]{0.32\linewidth}
        \centering
        \includegraphics[width=\textwidth]{figs/helmholtz_exact.png}
    \end{subfigure} 
    
    \begin{subfigure}[t]{0.32\linewidth}
        \centering
        \includegraphics[width=\textwidth]{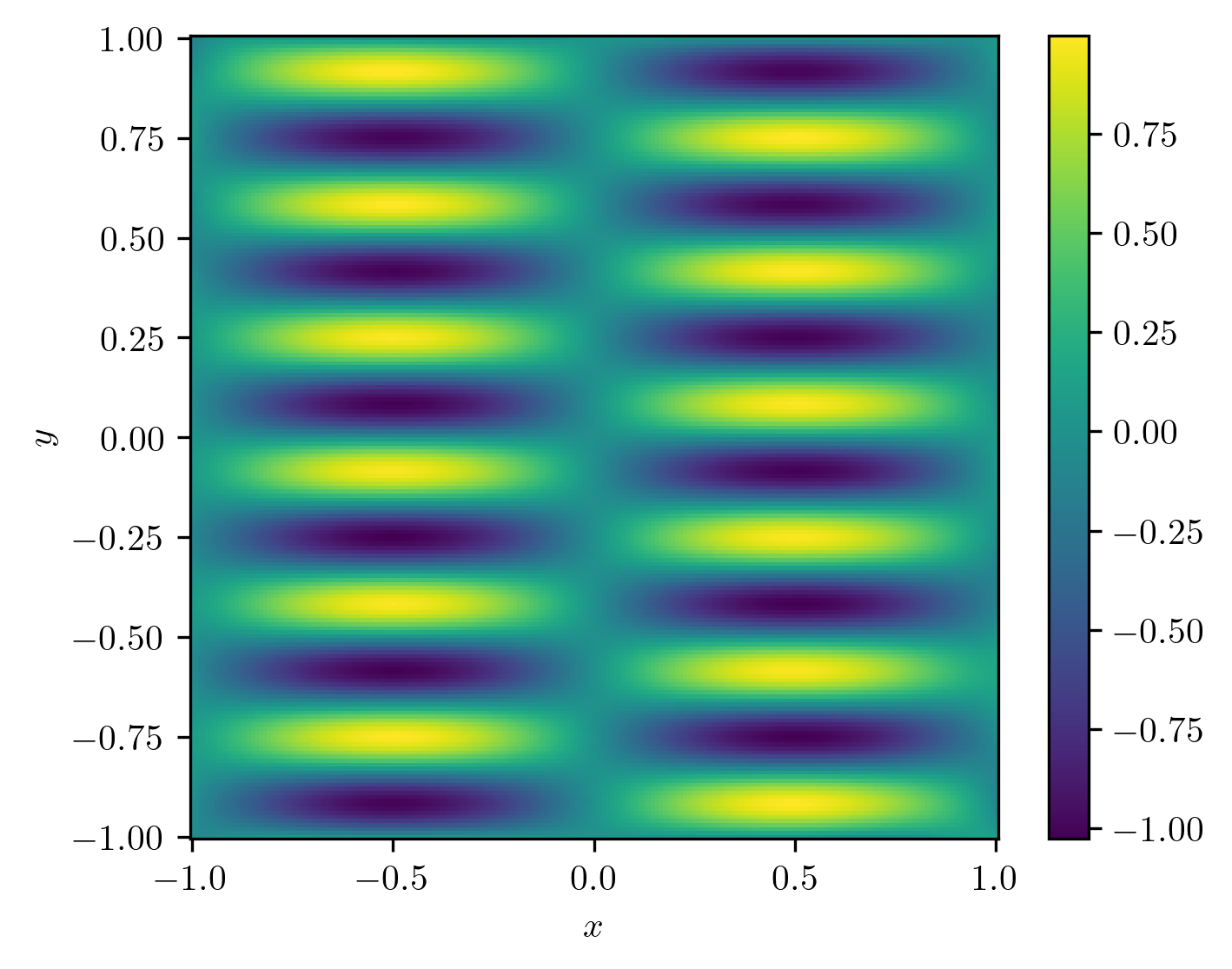}
    \end{subfigure}    
    \hfill
    \begin{subfigure}[t]{0.32\linewidth}
        \centering
        \includegraphics[width=\textwidth]{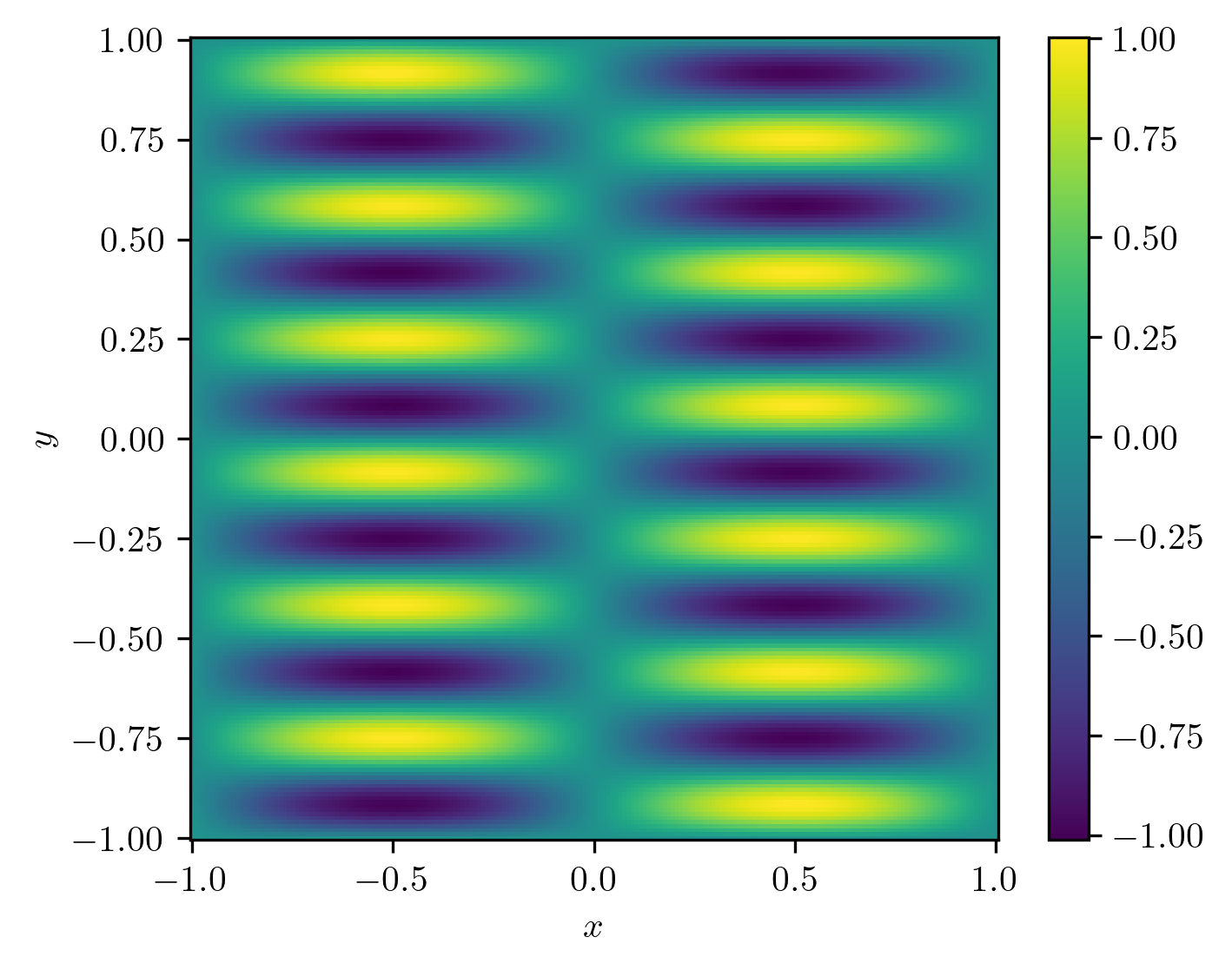}
    \end{subfigure}   
    \hfill
    \begin{subfigure}[t]{0.32\linewidth}
        \centering
        \includegraphics[width=\textwidth]{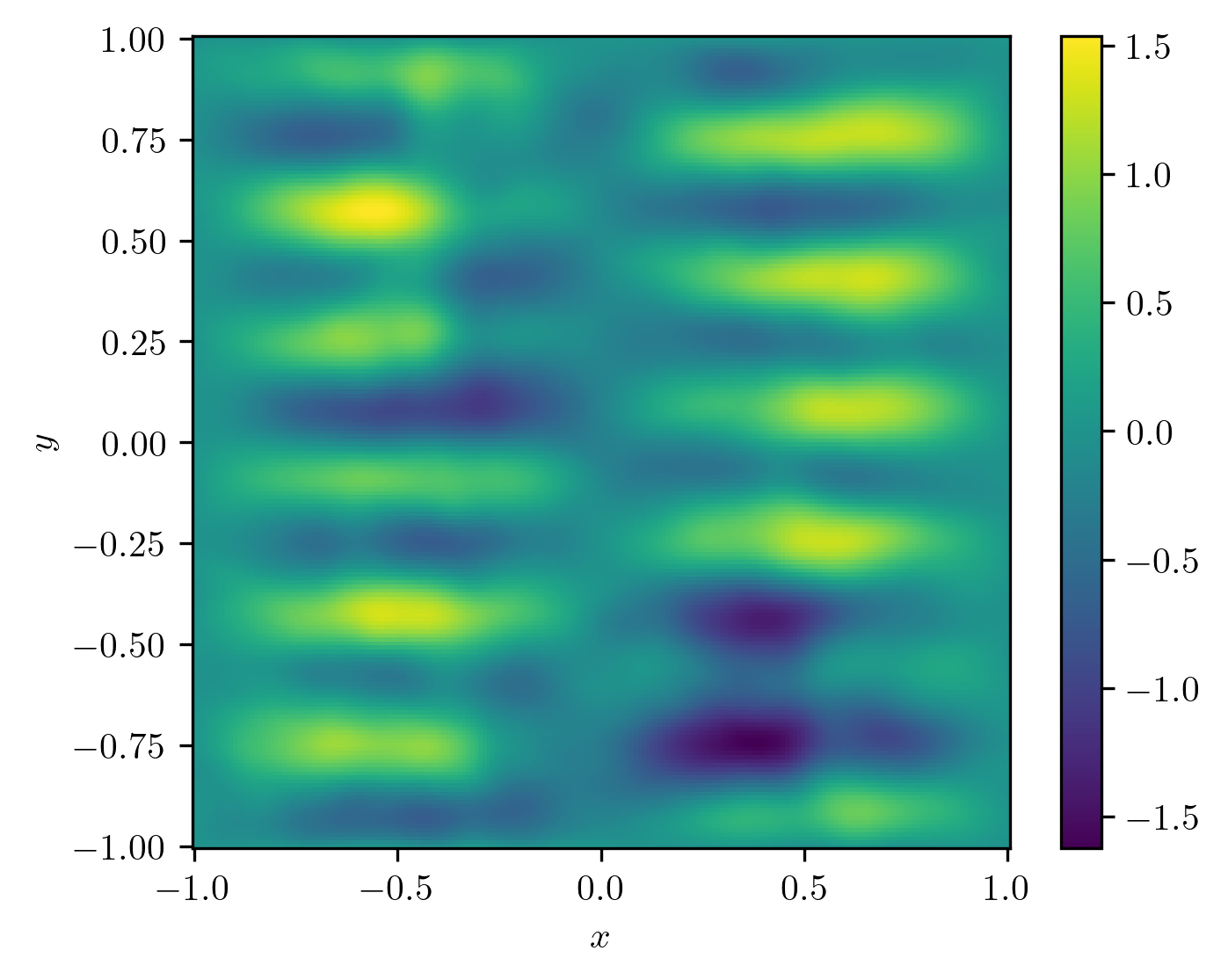}
    \end{subfigure} 

    \begin{subfigure}[t]{0.32\linewidth}
        \centering
        \includegraphics[width=\textwidth]{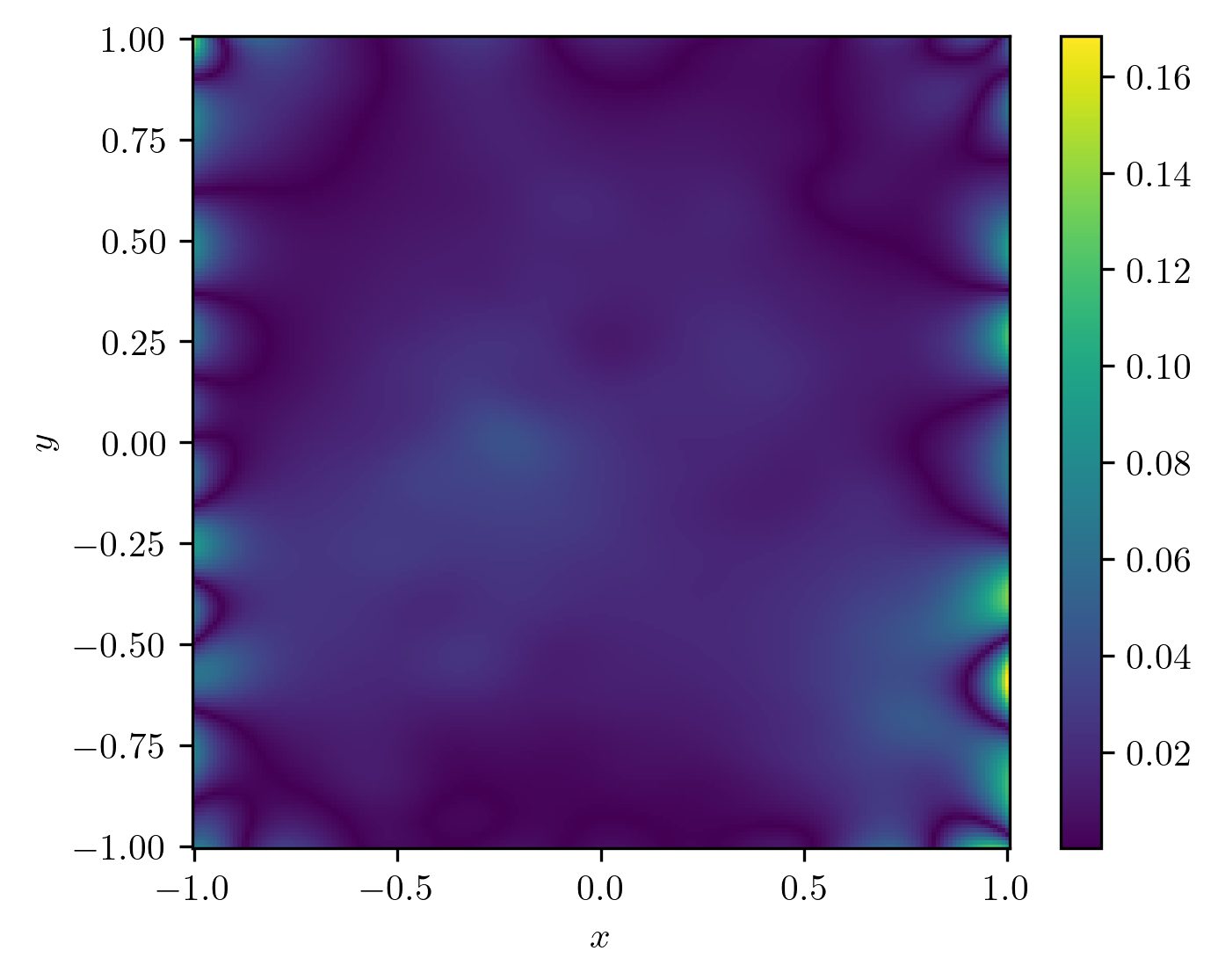}
        \caption{Sine 128-Neuron}
    \end{subfigure}    
    \hfill
    \begin{subfigure}[t]{0.32\linewidth}
        \centering
        \includegraphics[width=\textwidth]{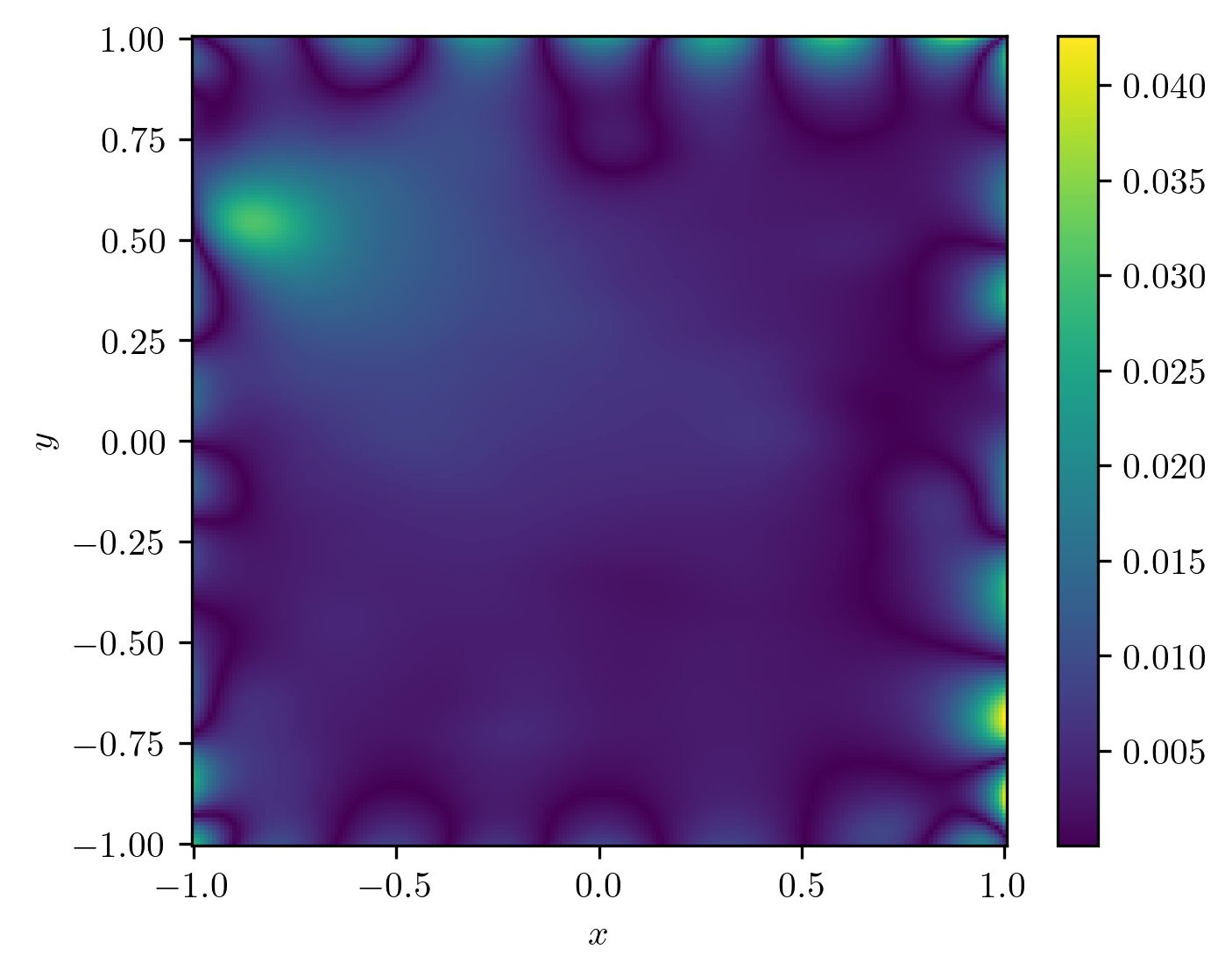}
        \caption{Sine 1024-Neuron}
    \end{subfigure}    
    \hfill
    \begin{subfigure}[t]{0.32\linewidth}
        \centering
        \includegraphics[width=\textwidth]{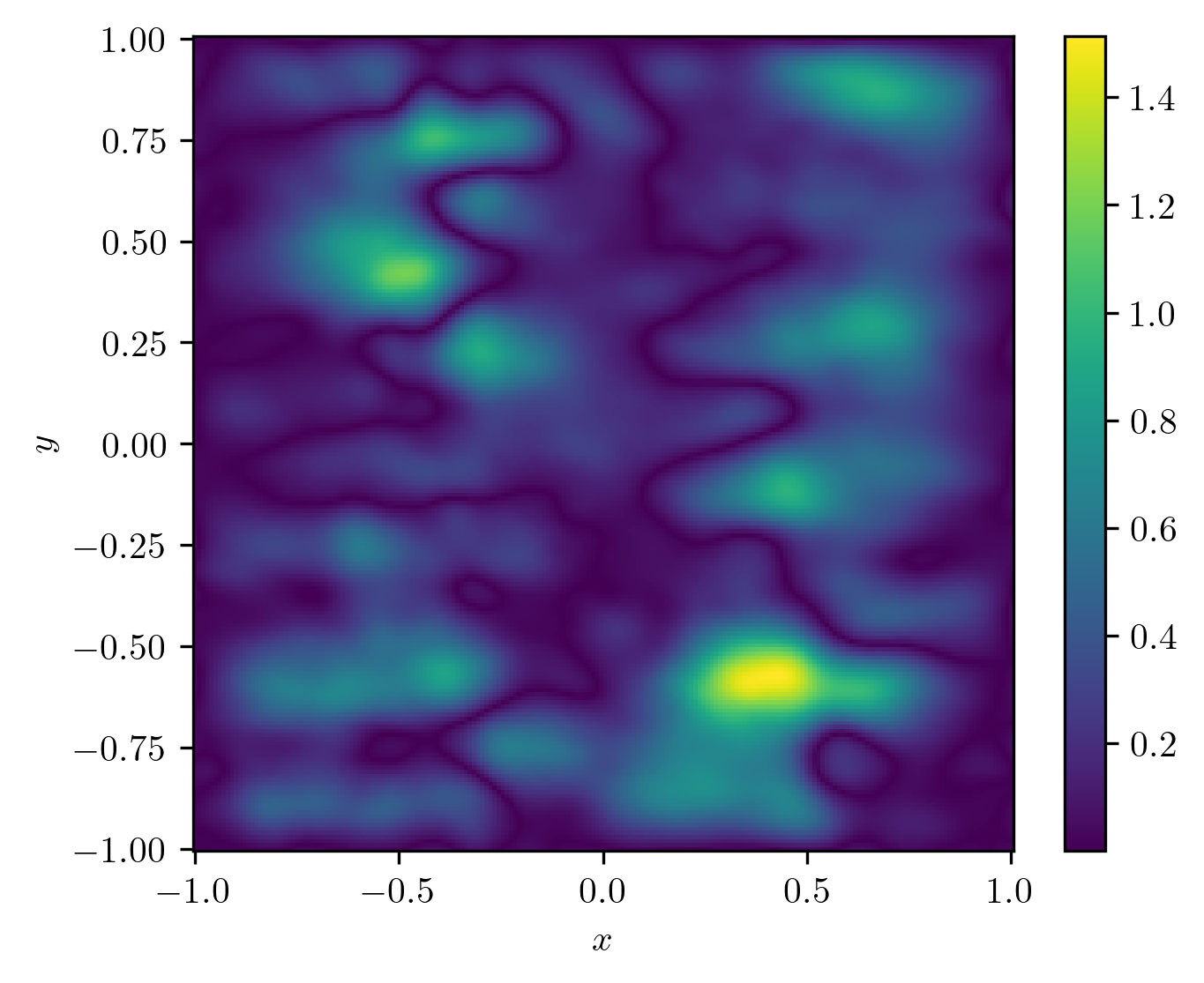}
        \caption{Tanh 1024-Neuron}
    \end{subfigure} 

    \caption{Helmholtz Equation. Top panels: Exact solution, Middle panels: Predicted solution, Bottom panels: Absolute Error}
    \label{fig: Helmholtz preds}
\end{figure}

\subsection{Klein-Gordon Equation}
\label{sec: KG}

\begin{figure}[t]
    \centering
    \begin{subfigure}[t]{0.32\linewidth}
    \centering
        \includegraphics[width=\textwidth]{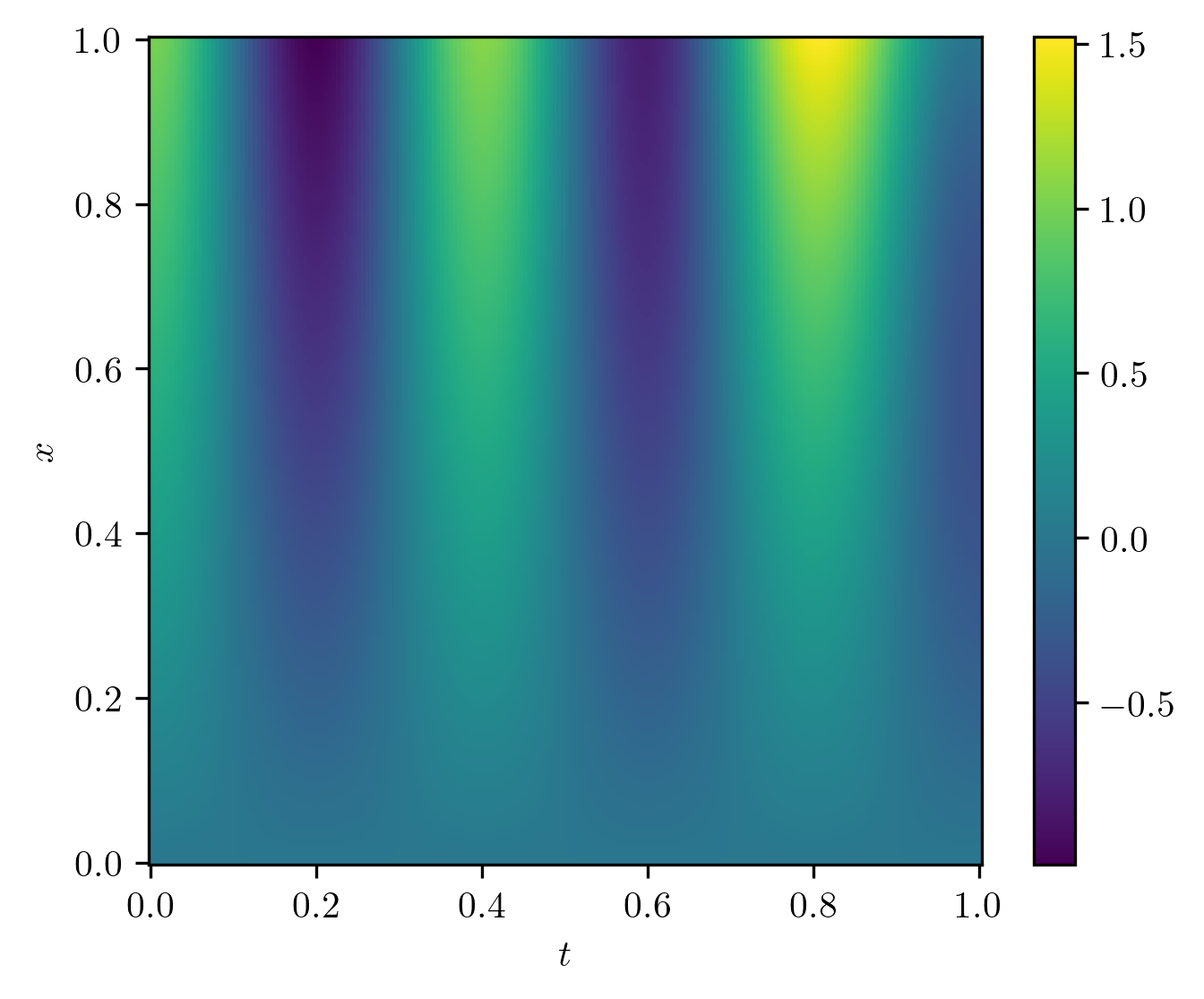}
        \caption{Exact Solution}
    \end{subfigure}
    \hfill
    \begin{subfigure}[t]{0.32\linewidth}
    \centering
        \includegraphics[width=\textwidth]{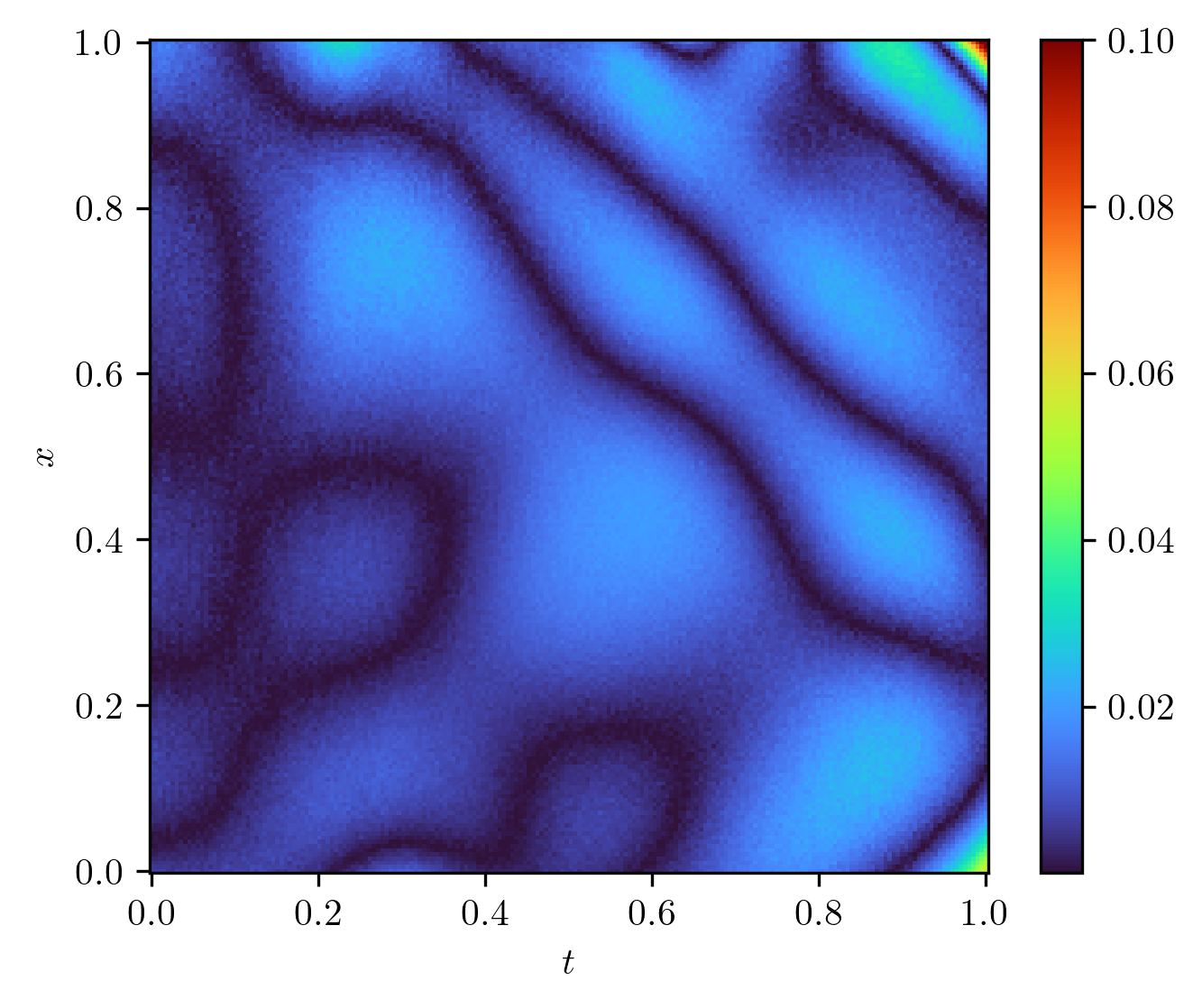}
        \caption{Tanh 128-Neuron}
    \end{subfigure}
    \hfill
    \begin{subfigure}[t]{0.32\linewidth}
    \centering
        \includegraphics[width=\textwidth]{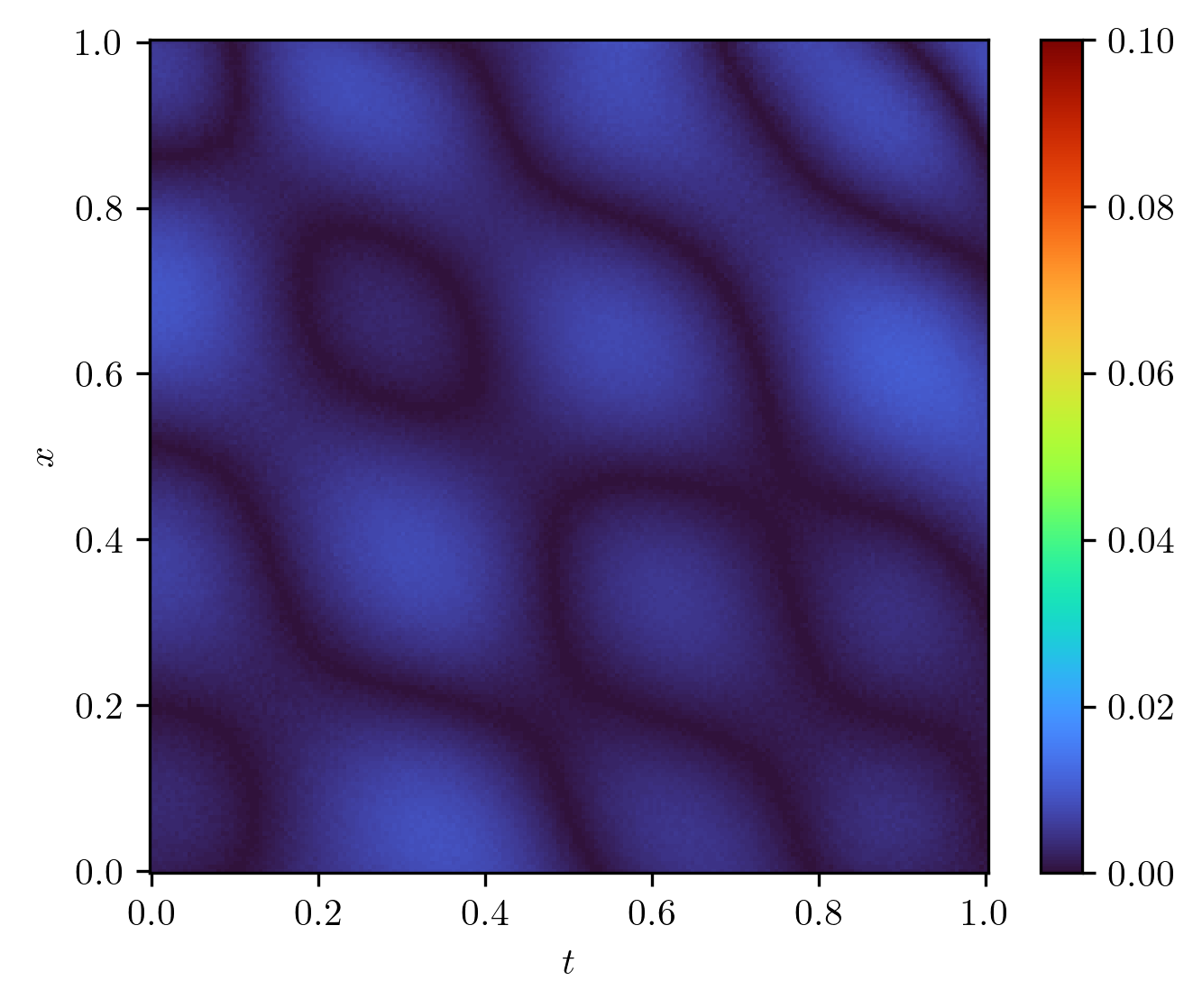}
        \caption{Sine 256-Neuron}
    \end{subfigure}
    \hfill
    \caption{Absolute error of the Klein-Gordon PINNs.}
    \label{fig:klein-gordon error}
\end{figure}

We conduct the same experiments on the non-linear one-dimensional Klein-Gordon equation of the following form:
\begin{align}
        &\frac{\partial^2 u}{\partial t^2} - \frac{\partial^2u}{\partial x^2} + u^3 = f(t, x),\;\;\;\;&x\in\left[0, 1\right],t\in\left[0, 1\right].\label{eq: Klein gordon}
\end{align}
We adopt the solution provided in \cite{wang2020understanding} and derive the source term $f(t,x)$ in Eq. \ref{eq: Klein gordon} to be consistent with the solution below:
\begin{align*}
    &u(t, x) = x\cos(5\pi t) + (xt)^3
\end{align*}
The model is then trained with a zero initial condition and a Dirichlet boundary condition corresponding to the solution $u$, using 256 collocation training points and 200 boundary data.

The resulting absolute errors are reported in Table \ref{tab: table second order}. Similar to other equations, the Sine activation function performs consistently better for all widths, achieving the best mean absolute error of $3.1\times 10^{-3}$ with the 256 neurons-wide model. Tanh also performs reasonably well with the 128 neurons-wide model, although it still performs slightly worse than the Sine models, as shown in Figure \ref{fig:klein-gordon error}.

\section{Conclusion}
The differentiation process in the residual loss of PINNs transforms the structure of the neural networks and their outputs, rendering the existing theory around loss functions and common supervised tasks ineffective in analyzing PINNs. In this work, we aim to fill the gap in our understanding of the residual loss and derive the requirements in network design that lead to achieving global minimization of this loss function. To this end, we study the residual loss at a critical point in the parameter space of the neural network and look for distinct characteristics of a global minimum that sets it apart from other critical points. We then use those characteristics to derive the requirements in the neural network design that ensure the existence of a global minimum. In particular, we show that under certain conditions, wide networks globally minimize the residual loss. Additionally, we reveal that activation functions with well-behaved high-order derivatives are crucial in the optimal minimization of the residual loss. We then use the established theory and empirical observations to choose activation functions and verify their effectiveness by conducting a set of experiments. The theory developed in this work paves the way for further development of better activation functions and provides a guideline for designing effective PINNs.

\section*{Acknowledgements}
This research is supported by Natural Sciences and Engineering Research Council of Canada (NSERC),
Discovery Grants program, and the Vector Scholarship in Artificial Intelligence, provided through the Vector Institute.

\bibliographystyle{named}
\bibliography{ijcai24}

\begin{thebibliography}{}

\bibitem[\protect\citeauthoryear{Allen-Zhu \bgroup \em et al.\egroup }{2019}]{allen2019convergence}
Zeyuan Allen-Zhu, Yuanzhi Li, and Zhao Song.
\newblock A convergence theory for deep learning via over-parameterization.
\newblock In {\em International conference on machine learning}, pages 242--252. PMLR, 2019.

\bibitem[\protect\citeauthoryear{Belbute-Peres and Kolter}{2022}]{belbute2022simple}
Filipe de~Avila Belbute-Peres and J~Zico Kolter.
\newblock Simple initialization and parametrization of sinusoidal networks via their kernel bandwidth.
\newblock {\em arXiv preprint arXiv:2211.14503}, 2022.

\bibitem[\protect\citeauthoryear{Cai \bgroup \em et al.\egroup }{2021}]{cai2021physics}
Shengze Cai, Zhiping Mao, Zhicheng Wang, Minglang Yin, and George~Em Karniadakis.
\newblock Physics-informed neural networks (pinns) for fluid mechanics: A review.
\newblock {\em Acta Mechanica Sinica}, 37(12):1727--1738, 2021.

\bibitem[\protect\citeauthoryear{Chen \bgroup \em et al.\egroup }{2020a}]{chen2020physics}
Yuyao Chen, Lu~Lu, George~Em Karniadakis, and Luca Dal~Negro.
\newblock Physics-informed neural networks for inverse problems in nano-optics and metamaterials.
\newblock {\em Optics express}, 28(8):11618--11633, 2020.

\bibitem[\protect\citeauthoryear{Chen \bgroup \em et al.\egroup }{2020b}]{chen2020generalized}
Zixiang Chen, Yuan Cao, Quanquan Gu, and Tong Zhang.
\newblock A generalized neural tangent kernel analysis for two-layer neural networks.
\newblock {\em Advances in Neural Information Processing Systems}, 33:13363--13373, 2020.

\bibitem[\protect\citeauthoryear{Dong and Ni}{2021}]{dong2021method}
Suchuan Dong and Naxian Ni.
\newblock A method for representing periodic functions and enforcing exactly periodic boundary conditions with deep neural networks.
\newblock {\em Journal of Computational Physics}, 435:110242, 2021.

\bibitem[\protect\citeauthoryear{Du and Hu}{2019}]{du2019width}
Simon Du and Wei Hu.
\newblock Width provably matters in optimization for deep linear neural networks.
\newblock In {\em International Conference on Machine Learning}, pages 1655--1664. PMLR, 2019.

\bibitem[\protect\citeauthoryear{Du \bgroup \em et al.\egroup }{2019}]{du2019gradient}
Simon Du, Jason Lee, Haochuan Li, Liwei Wang, and Xiyu Zhai.
\newblock Gradient descent finds global minima of deep neural networks.
\newblock In {\em International conference on machine learning}, pages 1675--1685. PMLR, 2019.

\bibitem[\protect\citeauthoryear{Farhani \bgroup \em et al.\egroup }{2022}]{farhani2022momentum}
G.~Farhani, Alexander Kazachek, and Boyu Wang.
\newblock Momentum diminishes the effect of spectral bias in physics-informed neural networks.
\newblock {\em arXiv preprint arXiv:2206.14862}, 2022.

\bibitem[\protect\citeauthoryear{Jacot \bgroup \em et al.\egroup }{2018}]{jacot2018neural}
Arthur Jacot, Franck Gabriel, and Cl{\'e}ment Hongler.
\newblock Neural tangent kernel: Convergence and generalization in neural networks.
\newblock {\em Advances in neural information processing systems}, 31, 2018.

\bibitem[\protect\citeauthoryear{Krishnapriyan \bgroup \em et al.\egroup }{2021}]{krishnapriyan2021characterizing}
Aditi Krishnapriyan, Amir Gholami, Shandian Zhe, Robert Kirby, and Michael~W Mahoney.
\newblock Characterizing possible failure modes in physics-informed neural networks.
\newblock {\em Advances in Neural Information Processing Systems}, 34:26548--26560, 2021.

\bibitem[\protect\citeauthoryear{Lee \bgroup \em et al.\egroup }{2019}]{lee2019wide}
Jaehoon Lee, Lechao Xiao, Samuel Schoenholz, Yasaman Bahri, Roman Novak, Jascha Sohl-Dickstein, and Jeffrey Pennington.
\newblock Wide neural networks of any depth evolve as linear models under gradient descent.
\newblock {\em Advances in neural information processing systems}, 32, 2019.

\bibitem[\protect\citeauthoryear{Li \bgroup \em et al.\egroup }{2020}]{li2020fourier}
Zongyi Li, Nikola Kovachki, Kamyar Azizzadenesheli, Burigede Liu, Kaushik Bhattacharya, Andrew Stuart, and Anima Anandkumar.
\newblock Fourier neural operator for parametric partial differential equations.
\newblock {\em arXiv preprint arXiv:2010.08895}, 2020.

\bibitem[\protect\citeauthoryear{Liu \bgroup \em et al.\egroup }{2020}]{liu2020linearity}
Chaoyue Liu, Libin Zhu, and Misha Belkin.
\newblock On the linearity of large non-linear models: when and why the tangent kernel is constant.
\newblock {\em Advances in Neural Information Processing Systems}, 33:15954--15964, 2020.

\bibitem[\protect\citeauthoryear{McClenny and Braga-Neto}{2020}]{mcclenny2020self}
Levi McClenny and Ulisses Braga-Neto.
\newblock Self-adaptive physics-informed neural networks using a soft attention mechanism.
\newblock {\em arXiv preprint arXiv:2009.04544}, 2020.

\bibitem[\protect\citeauthoryear{Meronen \bgroup \em et al.\egroup }{2021}]{meronen2021periodic}
Lassi Meronen, Martin Trapp, and Arno Solin.
\newblock Periodic activation functions induce stationarity.
\newblock {\em Advances in Neural Information Processing Systems}, 34:1673--1685, 2021.

\bibitem[\protect\citeauthoryear{Milnor \bgroup \em et al.\egroup }{1965}]{a54a6932-695b-3fed-b1c3-8acbcfb5005c}
John Milnor, L.~Siebenmann, and J.~Sondow.
\newblock {\em Lectures on the H-Cobordism Theorem}.
\newblock Princeton University Press, 1965.

\bibitem[\protect\citeauthoryear{Nguyen and Hein}{2017}]{nguyen2017loss}
Quynh Nguyen and Matthias Hein.
\newblock The loss surface of deep and wide neural networks.
\newblock In {\em International conference on machine learning}, pages 2603--2612. PMLR, 2017.

\bibitem[\protect\citeauthoryear{Nguyen and Mondelli}{2020}]{nguyen2020global}
Quynh~N Nguyen and Marco Mondelli.
\newblock Global convergence of deep networks with one wide layer followed by pyramidal topology.
\newblock {\em Advances in Neural Information Processing Systems}, 33:11961--11972, 2020.

\bibitem[\protect\citeauthoryear{Oymak and Soltanolkotabi}{2020}]{oymak2020toward}
Samet Oymak and Mahdi Soltanolkotabi.
\newblock Toward moderate overparameterization: Global convergence guarantees for training shallow neural networks.
\newblock {\em IEEE Journal on Selected Areas in Information Theory}, 1(1):84--105, 2020.

\bibitem[\protect\citeauthoryear{Rahaman \bgroup \em et al.\egroup }{2018}]{rahaman2018spectral}
Nasim Rahaman, A.~Baratin, Devansh Arpit, Felix Dräxler, Min Lin, F.~Hamprecht, Yoshua Bengio, and Aaron~C. Courville.
\newblock On the spectral bias of neural networks.
\newblock {\em International Conference On Machine Learning}, 2018.

\bibitem[\protect\citeauthoryear{Raissi \bgroup \em et al.\egroup }{2017}]{raissi2017physics}
Maziar Raissi, Paris Perdikaris, and George~Em Karniadakis.
\newblock Physics informed deep learning (part i): Data-driven solutions of nonlinear partial differential equations.
\newblock {\em arXiv preprint arXiv:1711.10561}, 2017.

\bibitem[\protect\citeauthoryear{Reiser \bgroup \em et al.\egroup }{2022}]{reiser2022graph}
Patrick Reiser, Marlen Neubert, Andr{\'e} Eberhard, Luca Torresi, Chen Zhou, Chen Shao, Houssam Metni, Clint van Hoesel, Henrik Schopmans, Timo Sommer, et~al.
\newblock Graph neural networks for materials science and chemistry.
\newblock {\em Communications Materials}, 3(1):93, 2022.

\bibitem[\protect\citeauthoryear{Safran and Shamir}{2016}]{safran2016quality}
Itay Safran and Ohad Shamir.
\newblock On the quality of the initial basin in overspecified neural networks.
\newblock In {\em International Conference on Machine Learning}, pages 774--782. PMLR, 2016.

\bibitem[\protect\citeauthoryear{Sirignano and Spiliopoulos}{2018}]{sirignano2018dgm}
Justin Sirignano and Konstantinos Spiliopoulos.
\newblock Dgm: A deep learning algorithm for solving partial differential equations.
\newblock {\em Journal of computational physics}, 375:1339--1364, 2018.

\bibitem[\protect\citeauthoryear{Sitzmann \bgroup \em et al.\egroup }{2020}]{sitzmann2020implicit}
Vincent Sitzmann, Julien Martel, Alexander Bergman, David Lindell, and Gordon Wetzstein.
\newblock Implicit neural representations with periodic activation functions.
\newblock {\em Advances in neural information processing systems}, 33:7462--7473, 2020.

\bibitem[\protect\citeauthoryear{Wang \bgroup \em et al.\egroup }{2020a}]{wang2020understanding}
Sifan Wang, Yujun Teng, and Paris Perdikaris.
\newblock Understanding and mitigating gradient pathologies in physics-informed neural networks.
\newblock {\em arXiv preprint arXiv:2001.04536}, 2020.

\bibitem[\protect\citeauthoryear{Wang \bgroup \em et al.\egroup }{2020b}]{wang2020pinns}
Sifan Wang, Xinling Yu, and P.~Perdikaris.
\newblock When and why pinns fail to train: A neural tangent kernel perspective.
\newblock {\em Journal Of Computational Physics}, 2020.

\bibitem[\protect\citeauthoryear{Wang \bgroup \em et al.\egroup }{2021}]{wang2021eigenvector}
Sifan Wang, Hanwen Wang, and Paris Perdikaris.
\newblock On the eigenvector bias of fourier feature networks: From regression to solving multi-scale pdes with physics-informed neural networks.
\newblock {\em Computer Methods in Applied Mechanics and Engineering}, 384:113938, 2021.

\bibitem[\protect\citeauthoryear{Wang \bgroup \em et al.\egroup }{2022}]{wang2022respecting}
Sifan Wang, Shyam Sankaran, and Paris Perdikaris.
\newblock Respecting causality is all you need for training physics-informed neural networks.
\newblock {\em arXiv preprint arXiv:2203.07404}, 2022.

\bibitem[\protect\citeauthoryear{Wight and Zhao}{2020}]{wight2020solving}
Colby~L Wight and Jia Zhao.
\newblock Solving allen-cahn and cahn-hilliard equations using the adaptive physics informed neural networks.
\newblock {\em arXiv preprint arXiv:2007.04542}, 2020.

\bibitem[\protect\citeauthoryear{Wong \bgroup \em et al.\egroup }{2022}]{wong2022learning}
Jian~Cheng Wong, Chinchun Ooi, Abhishek Gupta, and Yew-Soon Ong.
\newblock Learning in sinusoidal spaces with physics-informed neural networks.
\newblock {\em IEEE Transactions on Artificial Intelligence}, 2022.

\bibitem[\protect\citeauthoryear{Zhang \bgroup \em et al.\egroup }{2022}]{zhang2022physics}
Jie Zhang, Yihui Zhao, Fergus Shone, Zhenhong Li, Alejandro~F Frangi, Sheng~Quan Xie, and Zhi-Qiang Zhang.
\newblock Physics-informed deep learning for musculoskeletal modeling: Predicting muscle forces and joint kinematics from surface emg.
\newblock {\em IEEE Transactions on Neural Systems and Rehabilitation Engineering}, 31:484--493, 2022.

\end{thebibliography}
\newpage
\paragraph
{ }
\newpage
\appendix
\section{Generalizing to Multiple Independent Variables}
We first generalize Lemmas \ref{theorem: 2 layer Du} and \ref{theorem: 2 layer gradients W_L} and Theorem \ref{theorem: full rank two layer} in Section \ref{sec: section 3} to multiple ($d$) independent variables and present the proofs for them, and then, prove Theorem \ref{theorem: width N} in the next section. We consider a function of $d$ independent variables and a differential operator $\mathcal{D}$ with $k$-th order terms and coefficients $c_1, \dots, c_d$ of the following form:
\begin{equation}\label{eq: Du multi}
    \mathcal{D}\left[u\right] = c_1 \frac{\partial^k u}{\partial x_1^k} + \dots + c_d \frac{\partial^k u}{\partial x_d^k}
\end{equation}
We also define $(W_1)_i\in \mathbb R ^{1\times n_1}$ for $i \in \{1, \dots, d\}$ to represent the $i$-th row of the weight matrix $W_1$, i.e., $(W_1)_i$ are the weights corresponding to the $i$-th input variable.

\begin{lemma}[Generalization of Lemma \ref{theorem: 2 layer Du}] \label{lemma: multiple Du}
    For a two-layer neural network $\hat u$ of $d$ input variables and the $k$-th order differential operator in Eq. \ref{eq: Du multi}, $\mathcal{D}[\hat u]$ is given by
    \begin{equation*}
    \begin{aligned}
                \mathcal{D}[\hat{u}](x) = W_2^\top \times (F_1^{(k)}(x) \circ (&c_1(W_1)_1^k +\dots \\
                &+c_d (W_1)_d^k))^\top.
        \end{aligned}
    \end{equation*}
\end{lemma}
\begin{proof}
Note that the neural network is defined as
\begin{align*}
    \hat u (x) &= F_1(x)\times W_2 + b_2\\
    &=\sigma(G_1(x)) \times W_2 + b_2, \; G_1(x) = x \times W_1 + b_1
\end{align*}
    The first derivative $\frac{\partial \hat u}{\partial x_1}$ w.r.t. the first input variable is then derived as follows:
    \begin{align}
        \frac{\partial\hat u}{\partial x_1} &= \frac{\partial \hat u}{\partial F_1}\times \left(\frac{\partial F_1}{\partial G_1} \frac{\partial G_1}{\partial x_1}\right)\nonumber\\
        &= W_2^\top \times \left(\frac{\partial F_1}{\partial G_1} \frac{\partial G_1}{\partial x_1}\right)\label{eq: deriv F1}
    \end{align}
    Derivative of $F_1$ w.r.t. $G_1$ is just $\sigma'(G_1(x))$. We also have
    \begin{align*}
        G_1(x) &= x\times W_1 + b_1 \\
        &= \sum_{i=1}^{d} x_i(W_1)_i + b_1.
    \end{align*}
    Thus, Eq. \ref{eq: deriv F1} yields:
    \begin{align*}
        \frac{\partial \hat u}{\partial x_1} = W_2^\top \times \left(\sigma'(G_1(x))\circ (W_1)_1\right)^\top
    \end{align*}
    For higher-order derivatives, only $G_1(x)$ depends on $x_1$, and other terms are constants. Hence, assuming that the $(k-1)$-th derivative is given by $W_2^\top \times \left(\sigma^{(k-1)}(G_1(x))\circ (W_1)_1^{k-1}\right)^\top$, the $k$-th derivative is
    \begin{align*}
        \frac{\partial^k \hat u}{\partial x_1^k} &= W_2^\top \times \left(\frac{\partial F_1^{(k-1)}}{\partial x_1}\circ (W_1)^{k-1}_1\right)^\top\\
        &=W_2^\top \times \left(\sigma^{(k)}(G_1(x)) \circ (W_1)_1\circ (W_1)^{k-1}_1\right)^\top\\
        &=W_2^\top \times \left(F_1^{(k)}(x) \circ (W_1)^{k}_1\right)^\top.
    \end{align*}
    Derivatives w.r.t. other input variables are taken similarly. Thus, the differential operator $\mathcal{D}[\hat u]$ is
    \begin{align*}
        \mathcal{D}[\hat u](x) &= \sum_{i=1}^{d} c_i W_2^\top \times \left(F_1^{(k)}(x) \circ (W_1)^{k}_i\right)^\top\\
        &= W_2^\top \times \left(\sum_{i=1}^{d}c_i F_1^{(k)}(x) \circ (W_1)_i^k\right)\\
        &= W_2^\top \times \left(F_1^{(k)}(x) \circ \sum_{i=1}^{d}c_i(W_1)_i^k\right).
    \end{align*}
\end{proof}
Note that Lemma \ref{lemma: multiple Du} (and other results in this section) can be further generalized to include mixed derivatives such as 
$$
\frac{\partial^2 \hat u}{\partial x_1 x_2} = W_2^\top \times (F_1^{(k)}(x) \circ (W_1)_1 \circ (W_1)_2)^\top,
$$
which we don't consider for simplicity.

\begin{lemma}[Generalization of Lemma \ref{theorem: 2 layer gradients W_L}]\label{lemma: gradients multiple var}
    For $\hat u$ and $\mathcal D[\hat u]$ as in Lemma \ref{lemma: multiple Du}, gradients of the residual loss w.r.t. the weights of the second layer over the training collocation data $\mathbf x$ of $N$ samples are given by
    \begin{align*}
        &\nabla_{W_2}\phi_r(\mathbf x; \mathcal W) =\\  &\;\;\left(\sum_{i=1}^{d}c_i(W_1)_i^k\right)\circ\left(l'(\mathcal D[\hat u](\mathbf x) - f(\mathbf x))^\top \times F^{(k)}_1(\mathbf x)\right).
    \end{align*}
\end{lemma}
\begin{proof}
    Let $A\in \mathbb R ^ {1\times n_1}$ be $\sum_{i=1}^{d}c_i(W_1)_i^k$. From Lemma \ref{lemma: multiple Du}, the residual loss over collocation samples $\mathbf x = [x^1, \dots, x^N]^\top$ is given by
    \begin{align*}
    \phi_r(\mathbf x; \mathcal W) &= \sum_{i=1}^{N} l\biggl(W_2^\top \times \bigl(F_1^{(k)}(x^i) \circ A\bigr) - f(x^i)\biggr)\\
    &= \sum_{i=1}^{N} l\biggl(\sum_{j=1}^{n_1}(W_2)_j \bigl(F_{1j}^{(k)}(x^i) A_j\bigr) - f(x^i)\biggr)\,,
    \end{align*}
    where $F_{1j}^{(k)}$ is the scalar output of neuron $j$ in the first layer fed to $\sigma^{(k)}$, and $A_j$ and $(W_2)_j$ are the $j$-th elements of the row and column vectors $A$ and $W_2$ (note that $W_2\in\mathbb R^{n_1 \times 1}$). Let $r_i\in \mathbb R$ denote the residual $\sum_{j=1}^{n_1}(W_2)_j (F_{1j}^{(k)}(x^i) A_j) - f(x^i)$ for a sample $x^i\in \mathbf x$. The partial derivative $\frac{\partial \phi_r}{\partial (W_2)_\alpha}$ for $\alpha\in\{1,\dots,n_1\}$ is then derived as
    \begin{align*}
        \frac{\partial \phi_r}{\partial (W_2)_\alpha}(\mathbf x; \mathcal W) &= \sum_{i=1}^{N}\frac{\partial l(r_i)}{\partial r_i} \frac{\partial r_i}{\partial (W_2)_\alpha}\\
        &= \sum_{i=1}^{N}l'(r_i)F_{1\alpha}^{(k)}(x^i)A_\alpha\\
        &= A_\alpha\sum_{i=1}^{N}l'(r_i)F_{1\alpha}^{(k)}(x^i),
    \end{align*}
    and the gradients w.r.t. $W_2$ are
    \begin{align*}
        \nabla_{W_2}\phi_r(\mathbf x; \mathcal W) &= \begin{bmatrix}
            A_1\sum_{i=1}^{N}l'(r_i)F_{11}^{(k)}(x^i)\\ \vdots\\ A_{n_1}\sum_{i=1}^{N}l'(r_i)F_{1n_1}^{(k)}(x^i)
        \end{bmatrix}^\top\\
        & = A \circ \begin{bmatrix}
            \sum_{i=1}^{N}l'(r_i)F_{11}^{(k)}(x^i)\\ \vdots\\ \sum_{i=1}^{N}l'(r_i)F_{1n_1}^{(k)}(x^i)
        \end{bmatrix}^\top\\
        & = A \circ \biggl(\sum_{i=1}^{N}l'(r_i)F_1^{(k)}(x^i)\biggr)\\
        & = A \circ \bigl(L'^\top\times F_1^{(k)}(\mathbf x)\bigr),
    \end{align*}
    where $L'\in \mathbb R^N$ is $[l'(r_1), \dots, l'(r_N)]^\top$.
    
\end{proof}

\begin{theorem}[Generalization of Theorem \ref{theorem: full rank two layer}]\label{theorem: full rank two layer multiple}
    For $\hat u$ and $\mathcal{D} [\hat u]$ as in Lemma \ref{lemma: multiple Du}, a critical point $\overline{\mathcal W}$ of the residual loss $\phi_r(\mathbf{x}, \mathcal W)$ is a global minimum if the following conditions are satisfied:
\begin{enumerate}
    \item The summation $\sum_{i=1}^{d}c_i(\overline{W}_1)_i^k$ of the weights of each input variable in the first layer is strictly non-zero, 
    \item $F^{(k)}_1$ has full row rank, i.e., $\text{rank}(F_1^{(k)}(\mathbf x)) = N$.
\end{enumerate}
\end{theorem}
\begin{proof}
    At a critical point $\overline{\mathcal{W}}$ of $\phi_r$, the gradients of $\phi_r$ w.r.t. the weights of the second layer are zero. Let $L' \in \mathbb{R}^{N}$ be $l'(\mathcal{D}[\hat u_{\overline{\mathcal W}}](\mathbf x) - f(\mathbf{x}))$ and $A\in \mathbb R^{1\times n_1}$ be $\sum_{i=1}^{d}c_i(\overline{W}_1)_i^k$. Lemma \ref{lemma: gradients multiple var} yields
    $$
    \nabla_{W_2}\phi_r(\mathbf{x}; \overline{\mathcal{W}}) = A\circ({L'}^\top\times F^{(k)}_1(\mathbf{x})) = 0.
    $$
    Since $A$ is strictly non-zero, ${L'}^\top\times F^{(k)}_1(\mathbf{x})$ must be zero. Furthermore, if $F^{(k)}_1(\mathbf{x})$ has full row rank, the null space of its transpose is zero, and thus, $L'$ must be zero for the above product to hold. Since $l$ is convex, $l(\mathcal{D}[\hat u_{\overline{\mathcal{W}}}](\mathbf x) - f(\mathbf x))$ and $\phi_r(\mathbf x; \overline{\mathcal{W}})$ are also zero, and $\overline{\mathcal{W}}$ globally minimizes the residual loss.
\end{proof}

\section{Proof of Theorem \ref{theorem: width N}}
In this section, we present a brief explanation of the proof for Theorem \ref{theorem: width N}. We refer the readers to \cite[Section 4]{nguyen2017loss} for complete and more detailed proof, while we provide the necessary modifications for PINNs here. Note that in \cite{nguyen2017loss} it is shown that a global minimum for a normal neural network is achieved when $[F_1(\mathbf x), \mathbf{1}_N]$ is full row rank, and then, a width of $n_1 \geq N-1$ is shown to be sufficient. For PINNs, we showed that $F_1^{(k)}(\mathbf x)$ should be full row rank for globally minimizing the residual loss and a width of $n_1 \geq N$ is sufficient. 

In summary, the proof works by first showing that a parameterization of the neural network that results in full row-rank $F_1^{(k)}(\mathbf x)$ does exist. Based on the fact that $F_1^{(k)}(\mathbf x)$ and determinants of its sub-matrices are real-analytic functions of the model parameters, it is then shown that for any set of parameters $\{W_1, b_1\}$ that $F_1^{(k)}(\mathbf x)$ is not full row-rank, there exists another set of parameters in a ball of any positive radius around $\{W_1, b_1\}$ that results in $\textit{rank}(F_1^{(k)}(\mathbf x)) = N$. Finally, using the implicit function theorem and the non-degeneracy of the critical point on $\{W_2, b_2\}$, it is shown that there exists another critical point in a ball of an arbitrarily small radius around the original critical point that results in full row rank $F_1^{(k)}(\mathbf x)$.

We start by modifying Lemma 4.3 in \cite{nguyen2017loss}, showing that $F_1^{(k)}(\mathbf x)$ can be full row-rank for some parameters in the parameter space when $n_1 \geq N$.

\begin{lemma}[Existence of full row-rank $F_1^{(k)}(\mathbf x)$]\label{lemma: exist full rank} With Assumptions \ref{assumption: 2 layer single var} holding and for a two-layer neural network with a width $n_1 \geq N$, there exists a parameterization that results in $\textit{rank}(F_1^{(k)}(\mathbf x))=N$
\end{lemma}

\begin{proof}
    Following the proof in \cite{nguyen2017loss}, we construct the parameters $\{W_1, b_1\}$ such that $F_1^{(k)}$ is full row-rank. Since $n_1 \geq N$ and $F_1^{(k)}(\mathbf x) \in \mathbb R ^ {N\times n_1}$, $N$ columns from $F_1^{(k)}$ must be linearly independent for it to have rank $N$. We denote the sub-matrix containing the first $N$ columns with $A$ and the other $n_1 - N$ columns with $B$. Showing that the matrix $A$ is full rank for a set of parameters proves the argument as the original matrix $F_1^{(k)}(\mathbf x) = \left[A\; B\right]$ contains the same $N$ linearly independent columns in $A$.
    
    Pick a vector $a \in \mathbb R^d$ such that
    \begin{equation}\label{eq: ordering a}
        a\cdot x^1 < a\cdot x^2 < \dots < a\cdot x^N \;\;(a \cdot x^i \text{denotes dot product})    
    \end{equation}
    Note that by Assumption \ref*{assumption: 2 layer single var}.\ref{assumption: distinc samples}, the rows of the inputs $\mathbf x = [x^1, \dots, x^N]^\top$ are distinct. Thus, for some $i,j \in \{1, \dots, N\}$ that $i \neq j$, the set $\{a \in R^d \mid a\cdot(x^i - x^j) = 0\}$ is a hyperplane with Lebesgue measure zero. For all pairs of distinct samples in $\mathbf x$, there are $N(N-1) / 2$ such hyperplanes, and their union still has Lebesgue measure zero. As a result, a vector $a$ that results in the inequality above does exist.

    By Assumption \ref*{assumption: 2 layer single var}.\ref{assumption: infimum}, the infimum of $\sigma^{(k)}$ is zero. Also, let $\gamma$ denote the supremum of $\sigma^{(k)}$. Thus, $\lim_{x\to-\infty}\sigma^{(k)}(x) = 0$ and $\lim_{x\to\infty}\sigma^{(k)}(x) = \gamma$ as $\sigma^{(k)}$ is strictly monotonically increasing. Pick any $\beta \in \mathbb R$. Since $\sigma^{(k)}$ is strictly monotonically increasing, $\sigma^{(k)}(\beta)$ is greater than the infimum zero. For $\alpha \in \mathbb R$, the weights and biases of the first $N$ neurons in the first layer are defined as $w_j = -\alpha a$ and $b_j = \alpha x^j\cdot a + \beta$ for $j\in\{1, \dots, N\}$, as in \cite{nguyen2017loss}. With this choice of weights, the outputs $A$ of the first $N$ neurons become a function of $\alpha$ with $A(\alpha)_{ij} = \sigma^{(k)}(x^i\cdot w_j + b_j) = \sigma^{(k)}(\alpha(x^j - x^i)\cdot a + \beta)$ being the output of the $j$-th neuron for the $i$-th sample for $i,j\in\{1,\dots,N\}$. Based on the ordering in Eq.~\ref{eq: ordering a}, $A(\alpha)$ at the limit becomes:
    \begin{equation*}
    \resizebox{\linewidth}{!}{$
        \lim_{\alpha \to \infty}A(\alpha)=
        \begin{bmatrix}
        \sigma^{(k)}(\beta) & \gamma                     & \gamma                     & \dots     & \gamma      \\
        0                   & \sigma^{(k)}(\beta)   & \gamma                     & \dots     & \gamma      \\
        0                   & 0                     & \sigma^{(k)}(\beta)   & \dots     & \gamma      \\
        \vdots              & \vdots                & \vdots                & \ddots    & \vdots \\
        0                   & 0                     & 0                     & \dots     & \sigma^{(k)}(\beta)
        \end{bmatrix}
    $}
    \end{equation*}
    The matrix above is upper triangular with a positive diagonal, and thus, it has a positive determinant and is full rank.

    The rest of the proof is identical to \cite{nguyen2017loss}. As the determinant of $A(\alpha)$ is a polynomial function of its entries and is continuous in $\alpha$, for some $\alpha_0 \in \mathbb R$, $A(\alpha)$ has a positive determinant and is full rank for all $\alpha>\alpha_0$. Thus, the weights of the first $N$ neurons in the first layer can be chosen in a way that results in a full row rank $F_1^{(k)}(\mathbf x)$, regardless of the weights of other neurons. 
\end{proof}

The rest of the proof for Theorem \ref{theorem: width N} can be adapted from \cite[Section 4]{nguyen2017loss} by replacing $[F_1(\mathbf x), \mathbf 1_N]$ with $F_1^{(k)}(\mathbf x)$ and subtle modifications. Thus, we briefly go through the proof without formally presenting the details. The proof continues by observing that $F_1^{(k)}(\mathbf x)$ is a real analytic function of parameters $\{W_1, b_1\}$ (as long as $\sigma^{(k)}$ is real analytic), as it is a composition of linear transformations, multiplications, and addition, which are all real analytic, with the real analytic $\sigma^{(k)}$. Similarly, the determinants of the $N\times N$ sub-matrices of $F_1^{(k)}(\mathbf x)$ are also polynomial functions of the parameters and are real analytic. As proven in \cite[Lemma 4.4]{nguyen2017loss}, the last argument on determinants and Lemma \ref{lemma: exist full rank} conclude that the set $\{\{W_1, b_1\} \mid \textit{rank}(F_1^{(k)}(\mathbf{x}))< N \}$ has Lebesgue measure zero. This is because, for a real analytic function that is not constantly zero, the set of parameters that make it zero has Lebesgue measure zero. Since the determinant function is real analytic and is not constant zero, as shown in Lemma \ref{lemma: exist full rank}, the set of parameters that make it zero has Lebesgue measure zero. Consequently, for any set of parameters $\{W'_1, b'_1\}$ that $\textit{rank}(F_1^{(k)}(\mathbf{x}))< N$, the ball centred at $\{W'_1, b'_1\}$ with any radius $\epsilon>0$ contains at least one $\{W_1, b_1\}$ such that $\textit{rank}(F_1^{(k)}(\mathbf{x})) = N$ \cite[Corollary 4.5]{nguyen2017loss}.

Using the implicit function theorem, it is then shown that for a critical point $\overline{\mathcal W} = \{\overline W_1, \overline b_1, \overline W_2, \overline b_2\}$ that is non-degenerate on $\{W_2, b_2\}$, there exists an open ball around $\{\overline W_1, \overline b_1\}$ for some radius $\delta_1>0$ that for every $\{W'_1, b'_1\}$ in the ball, $\{W'_1, b'_1, \overline W_2, \overline b_2\}$ is also a critical point. Moreover, as we discussed before, for every $0<\delta_2\leq \delta_1$, there is a ball of radius $\delta_2$ around $\{W'_1, b'_1\}$ that contains at least one set of parameters $\{\Tilde W_1, \Tilde b_1\}$ that satisfies $\textit{rank}(F_1^{(k)}(\mathbf x)) = N$. Note that $\{\Tilde W_1, \Tilde b_1, \overline W_2, \overline b_2\}$ is still a critical point as well. Nguyen et al. \shortcite{nguyen2017loss} further show how $\{\Tilde W_1, \Tilde b_1\}$ are equal to the original $\{\overline W_1, \overline b_1\}$ as the radius of the ball approaches zero.

With the results above, we get a critical point that makes $F_1^{(k)}$ full-rank. Based on Theorems \ref{theorem: full rank two layer} and \ref{theorem: full rank two layer multiple}, the resulting critical point is also a global minimum as long as the condition on the sum of weights is also satisfied. Note that the rarity of degenerate critical points (as explained in Section \ref{sec: width N and activation}) means that non-degenerate ones are the common type of critical points when $n_1 \geq N$. Thus, the conditions on the sum of the weights of the first layer are also satisfied for most of the non-degenerate critical points with a high probability.

\section{Comparison with Sinusoidal Features}

\begin{table}[t]
\centering
\begin{tabular}{@{}c|ccccc@{}}
\toprule
    & TT    & TS    & W                 & K      & H \\ \midrule
128 & 0.613 & 0.011 & 0.406             & 0.024  & 0.590 \\
256 & 0.625 & 0.050 & \textbf{0.135}             & 0.071  & 0.762  \\
512 & 0.627 & 0.002 & \textbf{0.046}    & 0.089  & 0.502  \\
1024 & 0.654 & 0.668 & \textbf{0.024}   & 0.118  & 0.443  \\ \bottomrule
\end{tabular}%
\caption{MAE from sf-PINN. TT: Transport with Tanh, TS: Transport with Softplus, W: Wave, K: Klein-Gordon, H: Helmholtz}
\label{tab: sfpinn}
\end{table}

Applying random or trainable sinusoidal embeddings on PINN inputs has been explored before \cite{wong2022learning,wang2020pinns}. In particular, the first layer in sf-PINN uses $\sin(2\pi x)$ as the activation function, followed by other layers with any activation function of choice \cite{wong2022learning}. To confirm the effectiveness of the Sine activation function compared to Sine features in sf-PINN, we trained sf-PINNs for the equations in Section \ref{sec: experiments}. 

Table \ref{tab: sfpinn} reports the MAEs from sf-PINN for each equation. Compared to Tables \ref{tab: transport mae} and \ref{tab: table second order}, we observe that in most cases, our results outperform sf-PINN. Note that, consistent with Section \ref{sec: section 3}, Softplus activation significantly improves the performance of sf-PINNs for the Transport equation, underpinning our findings about the activation functions. In the case of the Wave equation, sf-PINN with Tanh performs better than plain Sine MLP, though the condition on width still improves the performance as the number of neurons grows beyond the number of collocation points ($N=512$). 

\section{Non-linear PDEs}

\begin{table}[t]
\centering
\begin{tabular}{@{}c|ccc|ccc@{}}
\toprule
 & \multicolumn{3}{c}{Navier-Stokes} & \multicolumn{3}{c}{Burger's} \\ \midrule
W   & T     & SF-T  & S      & T      & SF-T  & S    \\ \midrule
64  & 0.047 & 0.009 & 0.010  & 0.010  & 0.099 & 0.014 \\
128 & 0.010 & 0.010 & 0.011  & 0.013  & 0.046 & 0.034 \\
256 & 0.007 & 0.013 &  0.010 & 0.053  & 0.060 & 0.027 \\ \bottomrule
\end{tabular}
\caption{MAEs for Lid-driven Cavity ($Re=400$) and Burger's equation with vanilla Tanh MLP (T), sf-PINN with Tanh (SF-T), and Sine MLP (S).}
\label{tab: nonlinear}
\end{table}

Our theoretical analysis primarily focused on linear PDEs, and we empirically showed the effectiveness of our findings for solving linear PDEs. We also tried the non-linear Klein-Gordon equation in Section \ref{sec: KG} and demonstrated the effectiveness of our method for solving this equation. However, as apparent in Table \ref{tab: table second order}, the gains for the Klein-Gordon equation are not as significant and consistent as other linear equations, especially regarding the width.

To better understand to what extent our findings generalize to non-linear equations, we perform similar experiments with the non-linear Burger's equation and steady-state Navier-Stokes (lid-driven cavity flow with $Re=400$). Table \ref{tab: nonlinear} reports the MAEs for each equation. Similar to our observations from the Klein-Gordon equation, the gains from the width and Sine activation are not as consistent as linear equations. For the Navier-Stokes equation, while the Sine MLP performs consistently well across all various widths, the Tanh MLP eventually outperforms both sf-PINN and Sine MLP as the width grows. With Burger's equation, the Tanh MLP performs best with a narrow network and degrades as the width grows, while the Sine MLP fluctuates closely. 

Altogether, our experiments suggest that while the conditions discussed for linear equations can partly improve the performance of non-linear PINNs, non-linear equations come with intricacies regarding the representation power and optimization of PINNs that require further analysis.
\end{document}